%% file: main.tex
\documentclass{article}

\usepackage[nonatbib, final]{neurips_2019}

\usepackage{cite}

\usepackage[utf8]{inputenc} % allow utf-8 input
\usepackage[T1]{fontenc}    % use 8-bit T1 fonts
\usepackage[citecolor=blue]{hyperref}       % hyperlinks
\usepackage{url}            % simple URL typesetting
\usepackage{booktabs}       % professional-quality tables
\usepackage{amsfonts}       % blackboard math symbols
\usepackage{nicefrac}       % compact symbols for 1/2, etc.
\usepackage{microtype}      % microtypography

\usepackage{amsmath, amsthm, amssymb}
\allowdisplaybreaks

\usepackage{asymptote}
\usepackage{algorithm}
\usepackage[noend]{algpseudocode}
\usepackage[toc,page]{appendix}
\usepackage{tikz}
\usepackage{comment}
\usepackage{authblk}

\newtheorem{lemma}{Lemma}
\newtheorem{theorem}{Theorem}
\newtheorem{definition}{Definition}

\newtheorem{corollary}{Corollary}

\usepackage{xcolor}
\newcommand{\norm}[1]{\left\lVert#1\right\rVert}

\usepackage{hyperref}
\hypersetup{
    colorlinks=true,
    linkcolor=blue,
    filecolor=magenta,      
    urlcolor=cyan,
}

\DeclareMathOperator*{\argmin}{arg\,min}

\title{Beyond Online Balanced Descent: An Optimal Algorithm for Smoothed Online Optimization}

\author[*1]{Gautam Goel}
\author[*2,1]{Yiheng Lin}
\author[*1]{Haoyuan Sun}
\author[1]{Adam Wierman}
\affil[1]{California Institute of Technology} 
\affil[2]{Institute for Interdisciplinary Information Sciences, Tsinghua University}

\begin{document}

\renewcommand\footnotemark{}
\thanks{Gautam Goel, Yiheng Lin, and Haoyuan Sun contributed equally to this work. This work was supported by NSF grants  AitF-1637598 and CNS-1518941, with additional support for Gautam Goel provided by an Amazon AWS AI Fellowship.}

\maketitle

\begin{abstract}
  \input{Abstract.tex}

\end{abstract}

\section{Introduction}
\input{Introduction.tex}

\section{Model \& Preliminaries} 
\label{section:Pre}
\input{Model_Pre.tex}

\section{Lower Bounds} \label{section:LowerBounds}
\input{Gap_to_opt.tex}
\section{Algorithms} \label{section:NewAlgorithm}
\input{algorithm_intro.tex}
\subsection{Greedy OBD} \label{subsection:G-OBD}
\input{OBD_Plus.tex}
\subsection{Regularized OBD} \label{subsection:R-OBD}
\input{Main_Result.tex}

%\subsection{Discussion} \label{section:Discussion}
%\input{Discussion.tex}
\section{Balancing Regret and Competitive Ratio} \label{section:vs}
\input{Main_WG.tex}

\newpage

\bibliographystyle{abbrv}
\bibliography{main.bib}

%Commented out for NeurIPS Final submission

\newpage
\begin{appendices}

\input{Appendix.tex}
\end{appendices}
\end{document}

%% file: Abstract.tex
We study online convex optimization in a setting where the learner seeks to minimize the sum of a per-round hitting cost and a movement cost which is incurred when changing decisions between rounds.  We prove a new lower bound on the competitive ratio of any online algorithm in the setting where the costs are $m$-strongly convex and the movement costs are the squared $\ell_2$ norm. This lower bound shows that no algorithm can achieve a competitive ratio that is  $o(m^{-1/2})$ as $m$ tends to zero.  No existing algorithms have competitive ratios matching this bound, and we show that the state-of-the-art algorithm, Online Balanced Decent (OBD), has a competitive ratio that is $\Omega(m^{-2/3})$. We additionally propose two new algorithms, Greedy OBD (G-OBD) and Regularized OBD (R-OBD) and prove that both algorithms have an $O(m^{-1/2})$ competitive ratio. The result for G-OBD holds when the hitting costs are quasiconvex and the movement costs are the squared $\ell_2$ norm, while the result for R-OBD holds when the hitting costs are $m$-strongly convex and the movement costs are Bregman Divergences.  Further, we show that R-OBD simultaneously achieves constant, dimension-free competitive ratio and sublinear regret when hitting costs are strongly convex.

%\yiheng{I make a minor fix here. I change "quasi-convex" to "quasiconvex" to match the name in "Convex Optimization" book.}\adam{good catch}

%% file: Introduction.tex
%\adam{some of the references cite the arxiv version when the conference version is available.  We should fix that.  One example is our AIStats paper.}

We consider the problem of Smoothed Online Convex Optimization (SOCO), a variant of online convex optimization (OCO) where the online learner pays a movement cost for changing actions between rounds. More precisely, we consider a game where an online learner plays a series of rounds against an adaptive adversary. In each round, the adversary picks a convex cost function $f_t: \mathbb{R}^d \rightarrow \mathbb{R}_{\geq 0}$ and shows it to the learner. After observing the cost function, the learner chooses an action $x_t$ and pays a \textit{hitting cost} $f_t(x_t)$, as well as a \textit{movement cost} $c(x_t, x_{t-1})$, which penalizes the online learner for switching points between rounds. 

%This form of SOCO was originally proposed in the context of dynamic power management in data centers, and has since seen a wealth of applications, from speech animation to geographical load balancing \cite{kim2015decision}, \cite{joseph2012jointly}, \cite{kim2014real}, \cite{lin2012online}, and more recently applications in control \cite{goel2017thinking}, \cite{goel2018smoothed}, smart grid \adam{add cites}, online logistic regression \cite{goel2018smoothed}, and smoothed online maximum likelihood estimation \cite{goel2018smoothed}. 

SOCO was originally proposed in the context of dynamic power management in data centers \cite{lin2012online}.  Since then it has seen a wealth of applications, from speech animation to management of electric vehicle charging \cite{kim2015decision, joseph2012jointly, kim2014real}, and more recently applications in control \cite{goel2017thinking, goel2018smoothed} and power systems \cite{li2018using, badieionline}. SOCO has been widely studied in the machine learning community with the special cases of online logistic regression and smoothed online maximum likelihood estimation receiving recent attention \cite{goel2018smoothed}.

Additionally, SOCO has connections to a number of other important problems in online algorithms and learning.   Convex Body Chasing (CBC), introduced in \cite{friedman1993convex},  is a special case of SOCO \cite{bubeck2018competitively}. The problem of designing competitive algorithms for Convex Body Chasing has attracted much recent attention. e.g.  \cite{bubeck2018competitively, bansa2018nested, argue2019nearly}. SOCO can also be viewed as a continuous version of the Metrical Task System (MTS) problem (see \cite{borodin1992optimal, bartal1997polylog, blum2000line}). A special case of MTS is the celebrated $k-$server problem, first proposed in \cite{manasse1990competitive}, which has received significant attention in recent years (see \cite{bubeck2018k, buchbinder2019k}).

Given these connections, the design and analysis of algorithms for SOCO and related problems has received considerable attention in the last decade. SOCO was first studied in the scalar setting in \cite{lin2013dynamic}, which used SOCO to model dynamic ``right-sizing'' in data centers and gave a 3-competitive algorithm. A 2-competitive algorithm was shown in \cite{bansal20152}, also in the scalar setting, which matches the lower bound for online algorithms in this setting \cite{antoniadis2017tight}. Another rich line of work studies how to design competitive algorithms for SOCO when the online algorithm has access to predictions of future cost functions (see \cite{lin2012online, li2018using, chen2015online, chen2016using}).

Despite a large and growing literature on SOCO and related problems, for nearly a decade the only known constant-competitive algorithms that did not use predictions of future costs were for one-dimensional action spaces.  In fact, the connections between SOCO and Convex Body Chasing highlight that, in general, one cannot expect dimension-free constant competitive algorithms due to a $\Omega(\sqrt{d})$ lower bound (see \cite{friedman1993convex, chen2018smoothed}).  However, recently there has been considerable progress moving beyond the one-dimensional setting for large, important classes of hitting and movement costs. 

A breakthrough came in 2017 when \cite{chen2018smoothed} proposed a new algorithm, Online Balanced Descent (OBD), and showed that it is constant competitive in all dimensions in the setting where the hitting costs are locally polyhedral and movement costs are the $\ell_2$ norm. The following year, \cite{goel2018smoothed} showed that OBD is also constant competitive, specifically $3+O(1/m)$-competitive, in the setting where the hitting costs are $m$-strongly convex and the movement costs are the squared $\ell_2$ norm. Note that this setting is of particular interest because of its importance for online regression and LQR control (see \cite{goel2018smoothed}).  

While OBD has proven to be a promising new algorithm, at this point it is not known whether OBD is \emph{optimal} for the competitive ratio, or if there is more room for improvement. This is because there are no non-trivial lower bounds known for important classes of hitting costs, the most prominent of which is the class of strongly convex functions. 

\textbf{Contributions of this paper.} In this paper we prove the first non-trivial lower bounds on SOCO with strongly convex hitting costs, both for general algorithms and for OBD specifically.  These lower bounds show that OBD is not optimal and there is an order-of-magnitude gap between its performance and the general lower bound.  Motivated by this gap and the construction of the lower bounds we present two new algorithms, both variations of OBD, which have competitive ratios that match the lower bound.  More specifically, we make four main contributions in this paper.

First, we prove a new lower bound on the performance achievable by any online algorithm in the setting where the hitting costs are $m$-strongly convex and the movement costs are the squared $\ell_2$ norm. In particular, in Theorem \ref{GeneralLowerT1}, we show that as $m$ tends to zero, any online algorithm must have competitive ratio at least $\Omega(m^{-1/2})$. 

Second, we show that the state-of-the-art algorithm, OBD, cannot match this lower bound. More precisely, in Theorem \ref{OBDLowerT1} we show that, as $m$ tends to zero, the competitive ratio of OBD is $\Omega(m^{-2/3})$, an order-of-magnitude higher than the lower bound of $\Omega(m^{-1/2})$. This immediately begs the question: can any online algorithm close the gap and match the lower bound?

Our third contribution answers this question in the affirmative.  In Section \ref{section:NewAlgorithm}, we propose two novel algorithms, Greedy Online Balanced Descent (G-OBD) and Regularized Online Balanced Descent (R-OBD), which are able to close the gap left open by OBD and match the $\Omega(m^{-1/2})$ lower bound. 
Both algorithms can be viewed as ``aggressive" variants of OBD, in the sense that they chase the minimizers of the hitting costs more aggressively than OBD. In Theorem \ref{OBD_PLUS_T1} we show that G-OBD matches the lower bound up to constant factors for quasiconvex hitting costs (a more general class than $m$-strongly convex).  In Theorem \ref{MainT1} we show that R-OBD has a competitive ratio that \emph{precisely matches the lower bound, including the constant factors}, and hence can be viewed as an optimal algorithm for SOCO in the setting where the costs are $m$-strongly convex and the movement cost is the squared $\ell_2$ norm. Further, our results for R-OBD hold not only for squared $\ell_2$ movement costs; they also hold for movement costs that are Bregman Divergences, which commonly appear throughout information geometry, probability, and optimization.

Finally, in our last section we move beyond competitive ratio and additionally consider regret.  We prove in Theorem \ref{R-OBD-RegretT1} that R-OBD can simultaneously achieve bounded, dimension-free competitive ratio and sublinear regret in the case of $m$-strongly convex hitting costs and squared $\ell_2$ movement costs. This result helps close a crucial gap in the literature.  Previous work has shown that it not possible for any algorithm to simultaneously achieve both a constant competitive ratio and sublinear regret in general SOCO problems  \cite{daniely2019competitive}.  However, this was shown through the use of linear hitting and movement costs. Thus, the question of whether it is possible to simultaneously achieve a dimension-free, constant competitive ratio and sublinear regret when hitting costs are strongly convex has remained open.  The closest previous result is from \cite{chen2018smoothed}, which showed that OBD can achieve either constant competitive ratio or sublinear regret with locally polyhedral cost functions depending on the ``balance condition'' used; however both cannot be achieved simultaneously.  Our result (Theorem \ref{R-OBD-RegretT1}), shows that R-OBD can simultaneously provide a constant competitive ratio and sublinear regret for strongly convex cost functions when the movement costs are the squared $\ell_2$ norm.

%In this section, we describe the SOCO model and some application examples.

%SOCO has a large number of applications, for example, speech animation \cite{kim2015decision}, video streaming \cite{joseph2012jointly}, management of electric vehicle charging \cite{kim2014real}, geographical load balancing \cite{lin2012online},and multi-timescale control \cite{goel2017thinking}.

%SOCO can be viewed as a continuous version of Metrical Task System (MTS) problem (see \cite{borodin1992optimal}, \cite{bartal1997polylog}, \cite{blum2000line}). A special case of MTS problem is k-server problem (see \cite{manasse1990competitive}, \cite{bubeck2018k}).

%% file: Model_Pre.tex
%In this section, we formally introduce Smoothed Online Convex Optimization (SOCO), define the various performance metrics we use to compare online algorithms, and describe the previous state-of-the-art algorithm, Online Balanced Descent.

%\subsection{Smoothed Online Convex Optimization}

An instance of Smoothed Online Convex Optimization (SOCO) consists of a convex action set $\mathcal{X} \subset \mathbb{R}^d$, an initial point $x_0 \in \mathcal{X}$, a sequence of non-negative convex cost functions $f_1 \ldots f_t: \mathbb{R}^d \rightarrow \mathbb{R}_{\geq 0}$, and a movement cost $c: \mathbb{R}^d \times \mathbb{R}^d \rightarrow \mathbb{R}_{\geq 0}$. In every round, the environment picks a cost function $f_t$ (potentially adversarily) for an online learner. After observing the cost function, the learner chooses an action $x_t \in \mathbb{R}^d$ and pays a cost that is the sum of the \emph{hitting cost}, $f_t(x_t)$, and the \emph{movement cost}, a.k.a., switching cost, $c(x_t, x_{t-1})$.  The goal of the online learner is to minimize its total cost over $T$ rounds:
$cost(ALG) = \sum_{t=1}^T f_t(x_t) + c(x_t, x_{t-1}).$

We emphasize that it is the movement costs that make this problem interesting and challenging; if there were no movement costs, $c(x_t, x_{t-1}) = 0$, the problem would be trivial, since the learner could always pay the optimal cost simply by picking the action that minimizes the hitting cost in each round, i.e., by setting $x_t = \argmin_x f_t(x)$. The movement cost couples the cost the learner pays across rounds, which means that the optimal action of the learner depends on unknown future costs.

There is a long literature on SOCO, both focusing on algorithmic questions, e.g.,  \cite{goel2018smoothed, lin2013dynamic, bansal20152, chen2018smoothed}, and applications, e.g., \cite{kim2015decision, joseph2012jointly, kim2014real, lin2012online}. The variety of applications studied means that a variety of assumptions about the movement costs have been considered.  Motivated by applications to data center capacity management, movement costs have often been taken as the $\ell_1$ norm, i.e., $c(x_1, x_2) = \|x_1 - x_2 \|_1$, e.g. \cite{lin2013dynamic, bansal20152}.  However, recently, more general norms have been considered and the setting of squared $\ell_2$ movement costs has gained attention due to its use in online regression problems and connections to LQR control, among other applications (see \cite{goel2017thinking, goel2018smoothed, astrom2010feedback}).  

In this paper, we focus on the setting of the squared $\ell_2$ norm, i.e. $c(x_2, x_1) = \frac{1}{2} \|x_2 - x_1 \|_2^2$; however, we also consider a generalization of the $\ell_2$ norm in Section \ref{subsection:R-OBD} where $c$ is the Bregman divergence. Specifically, we consider $c(x_t, x_{t-1}) = D_h(x_t || x_{t-1}) = h(x_t) - h(x_{t-1}) - \langle \nabla h(x_{t-1}), x_t - x_{t-1} \rangle$, where both the potential $h$ and its Fenchel Conjugate $h^*$ are differentiable. Further, we assume that $h$ is $\alpha$-strongly convex and $\beta$-strongly smooth with respect to an underlying norm $\norm{\cdot}$. Definitions of each of these properties can be found in the appendix. 

Note that the squared $\ell_2$ norm is itself a Bregman divergence, with $\alpha = \beta = 1$ and $\norm{\cdot} = \norm{\cdot}_2$, $D_h(x_t || x_{t-1}) = \frac{1}{2}\norm{x_t - x_{t-1}}_2^2$.  However, more generally, when $h(y) = \sum_i y_i \ln y_i$ with domain $\Delta_n = \{y \in [0, 1]^n \mid \sum_i y_i = 1\}$, $D_h(x_t || x_{t-1})$ is the Kullback-Liebler divergence (see \cite{bansal2017potential}). Further,  $h$ is $\frac{1}{2\ln 2}$-strongly convex and $\frac{1}{\delta \ln 2}$-strongly smooth in the domain $\mathcal{X} = P_\delta = \{y \in [0, 1]^n \mid \sum_i y_i = 1, y_i \geq \delta\}$ (see \cite{chen2018smoothed}).  This extension is important given the role Bregman divergence plays across optimization and information theory, e.g., see \cite{azizan2018stochastic, murata2004information}. 

Like for movement costs, a variety of assumptions have been made about hitting costs.  In particular, because of the emergence of pessimistic lower bounds when general convex hitting costs are considered, papers typically have considered restricted classes of functions, e.g., locally polyhedral \cite{chen2018smoothed} and strongly convex \cite{goel2018smoothed}.  In this paper, we focus on hitting costs that are $m$-strongly convex; however our results in Section \ref{subsection:G-OBD} generalize to the case of quasiconvex functions.

\textbf{Competitive Ratio and Regret.} The primary goal of the SOCO literature is to design online algorithms that (nearly) match the performance of the offline optimal algorithm.  The performance metric used to evaluate an algorithm is typically the \textit{competitive ratio} because the goal is to learn in an environment that is changing dynamically and is potentially adversarial.  The competitive ratio is the worst-case ratio of total cost incurred by the online learner and the offline optimal costs.  The cost of the offline optimal is defined as the minimal cost of an algorithm if it has full knowledge of the sequence of costs $\{ f_t \}$, i.e.
$cost(OPT) = \min_{x_1 \ldots x_T}\sum_{t=1}^T f_t(x_t) + c(x_t, x_{t-1}).$ Using this, the \emph{competitive ratio} is defined as $ \sup_{f_1 \ldots f_T} cost(ALG)/cost(OPT).$ 

Note that another important performance measure of interest is the \emph{regret}. In this paper, we study a generalization of the classical regret called the $L$-constrained regret, which is defined as follows. The \emph{$L$-(constrained) dynamic regret} of an online algorithm $ALG$ is $\rho_L(T)$ if for all sequences of cost functions $f_t, \cdots, f_T$, we have $cost(ALG) - cost(OPT(L)) \leq \rho_L(T)$ where $OPT(L)$ is the cost of an $L$-constrained offline optimal solution, i.e., one with movement cost upper bounded by $L$: $OPT(L) = \min_{x \in \mathcal{X}^T}\sum_{t=1}^T f_t(x_t) + c(x_t, x_{t-1})\text{ subject to }\sum_{t=1}^T c(x_t, x_{t-1}) \leq L.$

As the definitions above highlight, the regret and competitive ratio both compare with the cost of an offline optimal solution, however regret constrains the movement allowed by the offline optimal.  The classical notion of regret focuses on the static optimal ($L=0$), but relaxing that to allow limited movement bridges regret and the competitive ratio since, as $L$ grows, the $L$-constrained offline optimal approaches the offline (dynamic) optimal.  Intuitively, one can think of regret as being suited for evaluating learning algorithms in (nearly) static settings while the competitive ratio as being suited for evaluating learning algorithms in dynamic settings. 

%In Section \ref{section:vs}, we also consider the performance of online algorithms with respect to \textit{L-constrained regret}. This is defined as the additive difference between the total cost incurred by the online learner and the optimal offline costs, where the offline optimal player is subject to a constraint that its movement costs cannot exceed $L$. More formally, the $L$-constrained regret is defined as $$ cost(ALG) - cost(OPT(L))$$ where 
%\begin{equation*}
%    cost(OPT(L)) = \min_{x\in X^T}\sum_{t=1}^T f_t(x_t) + d(x_t, x_{t-1}) \text{ subject to }\sum_{t=1}^T d(x_t, x_{t-1}) \leq L.
%\end{equation*}

\textbf{Online Balanced Descent.} The state-of-the-art algorithm for SOCO is Online Balanced Descent (OBD). OBD, which is formally defined in Algorithm \ref{alg:OBD}, uses the operator $\Pi_K(x): \mathbb{R}^d \to K$ to denote the $\ell_2$ projection of $x$ onto a convex set $K$; and this operator is defined as $\Pi_K(x) = \argmin_{y\in K}\norm{y - x}_2$.
Intuitively, it works as follows. In every round, OBD projects the previously chosen point $x_{t-1}$ onto a carefully chosen level set of the current cost function $f_t$. The level set is chosen so that the hitting costs and movement costs are ``balanced": in every round, the movement cost is at most a constant $\gamma$ times the hitting cost. The balance helps ensure that the online learner is matching the offline costs.  Since neither cost is too high, OBD ensures that both are comparable to the offline optimal. The parameter $\gamma$ can be tuned to give the optimal competitive ratio and the appropriate level set can be efficiently selected via binary search.

\begin{algorithm}[t]
\caption{Online Balanced Descent (OBD)}\label{alg:OBD}
\begin{algorithmic}[1]
\Procedure{OBD}{$f_t, x_{t-1}, \gamma$}\Comment{Procedure to select $x_t$}
\State $v_t \gets \argmin_x f_t(x)$
\State Let $x(l) = \prod_{K_t^l}(x_{t-1})$. Initialize $l = f_t(v_t)$. Here $K_t^l = \{x | f_t(x) \leq l\}$.
\State Increase $l$. Stop when $c(x(l), x_{t-1}) = \gamma (l - f_t(v_t))$.
\State $x_t \gets x(l)$.
\State \textbf{return} $x_t$
\EndProcedure
\end{algorithmic}
\end{algorithm}

Implicitly, OBD can be viewed as a proximal algorithm with a dynamic step size \cite{Boyd14proximal}, in the sense that, like proximal algorithms, OBD iteratively projects the previously chosen point onto a level set of the cost function. Unlike traditional proximal algorithms, OBD considers several different level sets, and carefully selects the level set in every round so as to balance the hitting and movement costs. We exploit this connection heavily when designing Regularized OBD (R-OBD), which is a proximal algorithm with a special regularization term added to the objective to help steer the online learner towards the hitting cost minimizer in each round. 

OBD was proposed in \cite{chen2018smoothed}, where the authors show that it has a constant, dimension-free competitive ratio in the setting where the movement costs are the $\ell_2$ norm and the hitting costs are locally polyhedral, i.e. grow at least linearly away from the minimizer. This was the first time an algorithm had been shown to be constant competitive beyond one-dimensional action spaces. In the same paper, a variation of OBD that uses a different balance condition was proven to have  $O(\sqrt{TL})$ $L$-constrained regret for locally polyhedral hitting costs.  OBD has since been shown to also have a constant, dimension-free competitive ratio when movement costs are the squared $\ell_2$ norm and hitting costs are strongly convex, which is the setting we consider in this paper. However, up until this paper, lower bounds for the strongly convex setting did not exist and it was not known whether the performance of OBD in this setting is optimal or if OBD can simultaneously achieve sublinear regret and a constant, dimension-free competitive ratio.  %These two results stand in stark contrast to the lower bounds in \cite{friedman1993convex}, which show that the competitive ratio of any online algorithm must scale exponentially with respect to dimension if the hitting costs are allowed to be arbitrary convex functions, i.e. are not restricted to be polyhedral or strongly convex. 

%% file: Gap_to_opt.tex
Our first set of results focuses on lower bounding the competitive ratio achievable by online algorithms for SOCO.  While \cite{chen2018smoothed} proves a general lower bound for SOCO showing that the competitive ratio of any online algorithm is $\Omega(\sqrt{d})$, where $d$ is the dimension of the action space, there are large classes of important problems where better performance is possible. In particular, when the hitting costs are $m$-strongly convex, \cite{goel2018smoothed} has shown that OBD provides a dimension-free competitive ratio of $3+O(1/m)$.  However, no non-trivial lower bounds are known for the strongly convex setting. 

Our first result in this section shows a general lower bound on the competitive ratio of SOCO algorithms when the hitting costs are strongly convex and the movement costs are quadratic.  Importantly, there is a gap between this bound and the competitive ratio for OBD proven in \cite{goel2018smoothed}.  Our second result further explores this gap.  We show a lower bound on the competitive ratio of OBD which highlights that OBD cannot achieve a competitive ratio that matches the general lower bound. This gap, and the construction used to show it, motivate us to propose new variations of OBD in the next section.  We then prove that these new algorithms have competitive ratios that match the lower bound. 

%\subsection{A general lower bound}

We begin by stating the first lower bound for strongly convex hitting costs in SOCO.  

\begin{theorem}\label{GeneralLowerT1}
Consider hitting cost functions that are $m$-strongly convex with respect to $\ell_2$ norm and movement costs given by $\frac{1}{2}\norm{x_t - x_{t-1}}_2^2$.  Any online algorithm must have a competitive ratio at least $\frac{1}{2}\left(1 +  \sqrt{1 + \frac{4}{m}} \right)$.
\end{theorem}

Theorem \ref{GeneralLowerT1} is proven in the appendix using an argument that leverages the fact that, when the movement cost is quadratic, reaching a target point via one large step is more costly than reaching it by taking many small steps. More concretely, to prove the lower bound we consider a scenario on the real line where the online algorithm encounters a sequence of cost functions whose minimizers are at zero followed by a  very steep cost function whose minimizer is at $x = 1$. Without knowledge of the future, the algorithm has no incentive to move away from zero until the last step, when it is forced to incur a large cost; however, the offline adversary, with full knowledge of the cost sequence, can divide the journey into multiple small steps.  

Importantly, the lower bound in Theorem \ref{GeneralLowerT1} highlights the dependence of the competitive ratio on $m$, the convexity parameter.  It shows that the case where online algorithms do the worst is when $m$ is small, and that algorithms that match the lower bound up to a constant are those for which the competitive ratio is $O(m^{-1/2})$ as $m\to0^+$.  Note that our results in Section \ref{section:NewAlgorithm} show that there exists online algorithms that precisely achieve the competitive ratio in Theorem \ref{GeneralLowerT1}. However, in contrast, the following shows that OBD cannot match the lower bound in Theorem \ref{GeneralLowerT1}. 

\begin{theorem}\label{OBDLowerT1}
Consider hitting cost functions that are $m$-strongly convex with respect to $\ell_2$ norm and a movement costs given by $\frac{1}{2}\norm{x_t - x_{t-1}}_2^2$. The competitive ratio of OBD is $\Omega(m^{-\frac{2}{3}})$ as $m\to 0^+$, for any fixed balance parameter $\gamma$.
\end{theorem}

As we have discussed, OBD is the state-of-the-art algorithm for SOCO, and has been shown to provide a competitive ratio of $3 + O\left(1/m\right)$ \cite{goel2018smoothed}. However, Theorem \ref{OBDLowerT1} highlights a gap between OBD and the general lower bound.  If the lower bound is achievable (which we prove it is in the next section), this implies that OBD is a sub-optimal algorithm.

The proof of Theorem \ref{OBDLowerT1} gives important intuition about what goes wrong with OBD and how the algorithm can be improved.  Specifically, our proof of Theorem \ref{OBDLowerT1} considers a scenario where the cost functions have minimizers very near each other, but OBD takes a series of steps without approaching the minimizing points. The optimal is able to pay little cost and stay near the minimizers, but OBD never moves enough to be close to the minimizers.  Figure \ref{figure:OBDLowerMainBody} illustrates the construction, showing OBD moving along the circumference of a circle, while the offline optimal stays near the origin.  %The key takeaway from this scenario is the observation that, in order to improve the competitive ratio of OBD, a variation should move more aggressively towards the minimizer in each time-step. This observation drives the design of the novel algorithms we propose in Section \ref{section:NewAlgorithm}.

\input{Figure_OBDLowerMainBody.tex}

%% file: Figure_OBDLowerMainBody.tex
\begin{figure}[t]
    \begin{center}
        \includegraphics{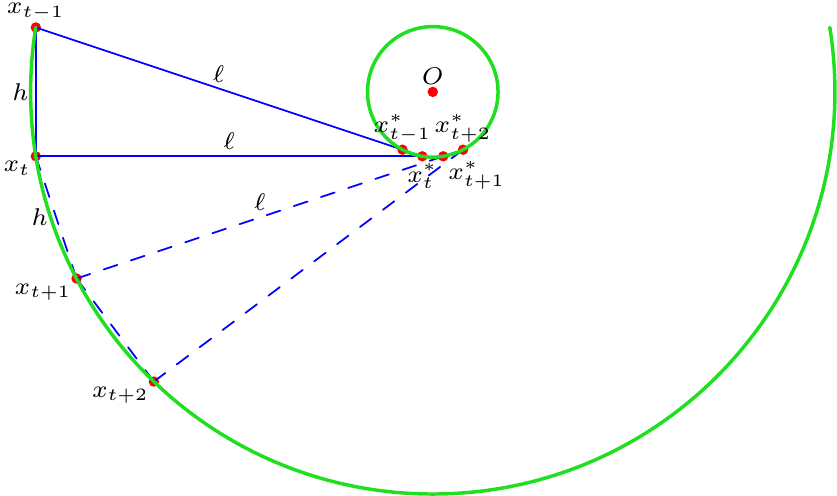}
    \end{center}
\caption{\emph{Counterexample used to prove Theorem \ref{OBDLowerT1}. In the figure, $\{x_t\}$ are the choices of OBD and $\{x_t^*\}$ are the choices of the offline optimal.}} %The hitting cost functions are constructed so that, at the end of every round, the relative positions of the OBD algorithm, the offline adversary, and the minimizer are fixed. Since OBD is memoryless, the function can be repeated and the positions of OBD and the offline adversary trace out a pair of concentric circles. \adam{Can we make this take up a little less space?}}}
\label{figure:OBDLowerMainBody}
\end{figure}

%% file: algorithm_intro.tex
The lower bounds in Theorem \ref{GeneralLowerT1} and Theorem \ref{OBDLowerT1} suggest a gap between the competitive ratio of OBD and what is achievable via an online algorithm.  Further, the construction used in the proof of Theorem \ref{OBDLowerT1} highlights the core issue that leads to inefficiency in OBD.  In the construction, OBD takes a large step from $x_{t-1}$ to $x_t$, but the offline optimal, $x_t^*$, only decreases by a very small amount. This means that OBD is continually chasing the offline optimal but never closing the gap. In this section, we take inspiration from this example and develop two new algorithms that build on OBD but ensure that the gap to the offline optimal $x_t^*$ shrinks.  

How to ensure that the gap to the offline optimal shrinks is not obvious since, without the knowledge about the future, it is impossible to determine how  $x_t^*$ will evolve.  A natural idea is to determine an online estimate of $x_t^*$ and then move towards that estimate.  Motivated by the construction in the proof of Theorem \ref{OBDLowerT1}, we use the minimizer of the hitting cost at round $t$, $v_t$, as a rough estimate of the offline optimal and ensure that we close the gap to $v_t$ in each round. 

There are a number of ways of implementing the goal of ensuring that OBD more aggressively moves toward the minimizer of the hitting cost each round.  In this section, we consider two concrete approaches, each of which (nearly) matches the lower bound in Theorem \ref{GeneralLowerT1}.

The first approach, which we term Greedy OBD (Algorithm \ref{alg:OBD_PLUS}) is a two-stage algorithm, where the first stage applies OBD and then a second stage explicitly takes a step directly towards the minimizer (of carefully chosen size).  We introduce the algorithm and analyze its performance in Section \ref{subsection:G-OBD}. Greedy OBD is order-optimal, i.e. matches the lower bound up to constant factors, in the setting of squared $\ell_2$ norm movement costs and quasiconvex hitting costs. 

The second approach for ensuring that OBD moves aggressively toward the minimizer uses a different view of OBD.  In particular, Greedy OBD uses a \emph{geometric view} of OBD, which is the way OBD has been presented previously in the literature.  Our second view uses a ``local view'' of OBD that parallels the \emph{local view} of gradient descent and mirror descent, e.g., see \cite{bansal2017potential, hazan2016introduction}.  In particular, the choice of an action in OBD can be viewed as the solution to a per-round \emph{local} optimization.  Given this view, we ensure that OBD more aggressively tracks the minimizer by adding a regularization term to this local optimization which penalizes points which are far from the minimizer.  We term this approach Regularized OBD (Algorithm \ref{alg:Breg}), and study it in Section \ref{subsection:R-OBD}.  Note that Regularized OBD has a competitive ratio that precisely matches the lower bound, including the constant factors, when movement costs are Bregman divergences and hitting costs are $m$-strongly convex.  Thus, it applies for more general movement costs than Greedy OBD but less general hitting costs. 

%% file: OBD_Plus.tex
The formal description of Greedy Online Balanced Descent (G-OBD) is given in Algorithm \ref{alg:OBD_PLUS}. G-OBD has two steps each round. First, the algorithm takes a standard OBD step from the previous point $x_{t-1}$ to a new point $x_t'$, which is the projection of $x_{t-1}$ onto a level set of the current hitting cost $f_t$, where the level set is chosen to balance hitting and movement costs. G-OBD then takes an additional step directly towards the minimizer of the hitting cost, $v_t$, with the size of the step chosen based on the convexity parameter $m$.  G-OBD can be implemented efficiently using the same approach as described for OBD \cite{chen2018smoothed}.  G-OBD has two parameters $\gamma$ and $\mu$.  The first, $\gamma$, is the balance parameter in OBD and the second, $\mu$, is a parameter controlling the size of the step towards the minimizer $v_t$.  Note that the two-step approach of G-OBD is reminiscent of the two-stage algorithm used in \cite{bienkowski2018better}; however the resulting algorithms are quite distinct.

\begin{algorithm}[t]
\caption{Greedy Online Balanced Descent (G-OBD)}\label{alg:OBD_PLUS}
\begin{algorithmic}[1]
\Procedure{G-OBD}{$f_t, x_{t-1}$}\Comment{Procedure to select $x_t$}
\State $v_t \gets \argmin_x f_t(x)$
\State $x_t' \gets OBD(f_t, x_{t-1}, \gamma)$
% \State (For convenience, define $H_t' := f_t(x_t')$ and $M_t' := c(x_t', x_{t-1})$.)
\If{$\mu \sqrt{m} \geq 1$}
\State $x_t \gets v_t$
\Else
\State $x_t \gets \mu \sqrt{m} v_t + (1 - \mu \sqrt{m})x_t'$
\EndIf
\State \textbf{return} $x_t$
\EndProcedure
\end{algorithmic}
\end{algorithm}

While the addition of a second step in G-OBD may seem like a small change, it improves performance by an order-of-magnitude.  We prove that G-OBD asymptotically matches the lower bound proven in Theorem \ref{OBDLowerT1} not just for $m$-strongly convex hitting costs, but more broadly to quasiconvex costs.

\begin{theorem}\label{OBD_PLUS_T1}
Consider quasiconvex hitting costs such that $f_t(x) \geq f_t(v_t) + \frac{m}{2}\norm{x - v_t}_2^2$ and movement costs $c(x_t, x_{t-1}) = \frac{1}{2}\norm{x_t - x_{t-1}}_2^2$. G-OBD with $\gamma = 1, \mu = 1$ is an $O\left(m^{-1/2}\right)$-competitive algorithm as $m\to 0^+$. \end{theorem}

%The proof of Theorem \ref{OBD_PLUS_T1} is provided in Appendix \adam{add ref}. 

\begin{comment}
Motivated by the proof of Theorem 1 in \cite{goel2018smoothed}, we use potential argument to prove case $H_t' \leq H_t^*$ and case $H_t' > H_t^*$ separately. In other words, let the potential be $\phi(x_t, x_t^*) = \eta \norm{x_t - x_t^*}_2^2$, we prove that
$$H_t + M_t + \Delta \phi \leq C(H_t^* + M_t^*),$$
for some positive constant $C$. From this inequality, we can sum over all timesteps $t$ to yield that the competitive ratio is upper bounded by $C$:
\[ \sum_{t=0}^T H_t + M_t \le \sum_{t=0}^T H_t + M_t + \Delta\phi \le C \sum_{t=0}^T\left( H^*_t + M^*_t\right).\]
This proof technique is also used in the proof of Theorem \ref{MainT1}.

Although Theorem \ref{OBD_PLUS_T1} only gives an upper bound for the competitive ratio of G-OBD when $m\leq \frac{9}{64}$, it is sufficient to show that G-OBD is optimal in the order of m when $m\to 0^+$.

\end{comment}

%% file: Main_Result.tex
The G-OBD framework is based on the geometric view of OBD used previously in literature.  There are, however, two limitations to this approach. First, the competitive ratio obtained, while having optimal asymptotic dependence on $m$, does not not match the constants in the lower bound of Theorem \ref{GeneralLowerT1}. Second, G-OBD requires repeated projections, which makes efficient implementation challenging when the functions $f_t$ have complex geometry.   

Here, we present a variation of OBD based on a \emph{local view} that overcomes these limitations.  Regularized OBD (R-OBD) is computationally simpler and provides a competitive ratio that matches the constant factors in the lower bound in Theorem \ref{GeneralLowerT1}.  However, unlike G-OBD, our analysis of R-OBD does not apply to quasiconvex hitting costs. R-OBD is described formally in Algorithm \ref{alg:Breg}.  In each round, R-OBD picks a point that minimizes a weighted sum of the hitting and movement costs, as well as a regularization term which encourages the algorithm to pick points close to the minimizer of the current hitting cost function, $v_t=\argmin_{x} f_t(x)$. Thus, R-OBD can be implemented efficiently using two invocations of a convex solver. Note that R-OBD has two parameters $\lambda_1$ and $\lambda_2$ which adjust the weights of the movement cost and regularizer respectively. 

While it may not be immediately clear how R-OBD connects to OBD, it is straightforward to illustrate the connection in the squared $\ell_2$ setting.  In this case, computing $x_t = \argmin_x f_t(x) + \frac{\lambda_1}{2}\norm{x - x_{t-1}}_2^2$ is equivalent to doing a projection onto a level set of $f_t$, since the selection of the minimizer can be restated as the solution to $\nabla f_t(x_t) + \lambda_1(x_t - x_{t-1}) = 0$. Thus, without the regularizer, the optimization in R-OBD gives a local view of OBD and then the regularizer provides more aggressive movement toward the minimizer of the hitting cost.  

%\begin{algorithm}
%\caption{Regularized OBD \adam{describe using $c(\cdot)$ not $D_h$.}}\label{alg:Meta}
%\begin{algorithmic}[1]
%\Procedure{R-OBD}{$f_t, x_{t-1}$}\Comment{Procedure to select $x_t$}
%\State $v_t \gets \argmin_x f_t(x)$
%\State $x_t \gets \argmin_x (\text{Hitting Cost}) + \lambda_1 (\text{Moving Cost}) + \lambda_2 (\text{Distance to }v_t)$
%\State \textbf{return} $x_t$
%\EndProcedure
%\end{algorithmic}
%\end{algorithm}

\begin{algorithm}[t]
\caption{Regularized OBD (R-OBD)}\label{alg:Breg}
\begin{algorithmic}[1]
\Procedure{R-OBD}{$f_t, x_{t-1}$}\Comment{Procedure to select $x_t$}
\State $v_t \gets \argmin_x f_t(x)$
\State $x_t \gets \argmin_x f_t(x) + \lambda_1 c(x, x_{t-1}) + \lambda_2 c(x, v_t)$
\State \textbf{return} $x_t$
\EndProcedure
\end{algorithmic}
\end{algorithm}

Not only does the local view lead to a computationally simpler algorithm, but we prove that R-OBD  matches the constant factors in Theorem \ref{GeneralLowerT1}  precisely, not just asymptotically.  Further, it does this not just in the setting where movement costs are the squared $\ell_2$ norm, but also in the case where movement costs are Bregman divergences.

\begin{theorem}\label{MainT1}
Consider hitting costs that are $m-$strongly convex with respect to a norm $\norm{\cdot}$ and movement costs defined as $c(x_t, x_{t-1}) = D_h(x_t || x_{t-1})$, where $h$ is $\alpha$-strongly convex and $\beta$-strongly smooth with respect to the same norm. Additionally, assume $\{f_t\}, h$ and its Fenchel Conjugate $ h^*$ are differentiable. Then, R-OBD with parameters $1 \geq \lambda_1 > 0$ and $\lambda_2 \geq 0$ has a competitive ratio of
$\max \left(\frac{m + \lambda_2 \beta}{\lambda_1}\cdot \frac{1}{m}, 1 + \frac{\beta^2}{\alpha}\cdot \frac{\lambda_1}{\lambda_2 \beta + m}\right).$
If $\lambda_1$ and $\lambda_2$ satisfy $m + \lambda_2 \beta =  \frac{\lambda_1 m}{2}\left(1 + \sqrt{1 + \frac{4\beta^2}{\alpha m}}\right)$ then the competitive ratio is 
$\frac{1}{2}\left(1 + \sqrt{1 + \frac{4\beta^2}{m\alpha}}\right).$

\end{theorem}

%\begin{theorem}\label{thm-Breg-cr}
%There exists a choice $\lambda_1, \lambda_2$ such that the competitive ratio of Regularized OBD (Algorithm \ref{alg:Breg}) in the Bregman Divergence setting (Problem Setting \ref{prob:Breg}) is at most $\frac{1}{2}\left(1 + \sqrt{1 + \frac{4\beta^2}{m\alpha}}\right)$.
%\end{theorem}

Theorem \ref{MainT1} focuses on movement costs that are Bregman divergences, which generalizes the case of squared $\ell_2$ movement costs.  To recover the squared $\ell_2$ case, we use $\norm{\cdot} = \norm{\cdot}_2$ and $\alpha = \beta = 1$, which results in a competitive ratio of $\frac{1}{2}(1 + \sqrt{1 + 4/m})$. This competitive ratio matches exactly with the lower bound claimed in Theorem \ref{GeneralLowerT1}.  Further, in this case the assumption in Theorem \ref{MainT1} that the hitting cost functions are differentiable is not required (see Theorem \ref{thm-optimal-cr} in the appendix).

It is also interesting to investigate the settings of $\lambda_1$ and $\lambda_2$ that yield the optimal competitive ratio. Setting $\lambda_2=0$ achieves the optimal competitive ratio as long as  $\lambda_1 = 2\left( 1 + \sqrt{1 + \frac{4\beta^2}{\alpha m}}\right)^{-1}$. By restating the update rule in R-OBD as $\nabla f_t(x_t) = \lambda_1 (\nabla h(x_{t-1}) - \nabla h(x_t))$, we see that R-OBD with $\lambda_2 = 0$ can be interpreted as ``one step lookahead mirror descent''. Further R-OBD with $\lambda_2=0$ can be implemented even when we do not know the location of the minimizer $v_t$. For example, when $h(x) = \frac{1}{2}\norm{x}_2^2$, we can run gradient descent starting at $x_{t-1}$ to minimize the strongly convex function $f_t(x) + \frac{\lambda_1}{2}\norm{x - x_{t-1}}_2^2$. Only local gradients will be queried in this process.  However, the following lower bound highlights that this simple form comes at some cost in terms of generality when compared with our results for G-OBD.

\begin{theorem}\label{WG_tradeT1}
Consider quasiconvex hitting costs such that $f_t(x) - f_t(v_t) \geq \frac{m}{2}\norm{x - v_t}_2^2$ and movement costs given by $c(x_t, x_{t-1}) = \frac{1}{2}\norm{x_t - x_{t-1}}_2^2$.  Regularized OBD has a competitive ratio of $\Omega(1/m)$ when $\lambda_2=0$.
\end{theorem}

%% file: Main_WG.tex
In the previous sections we have focused on the competitive ratio; however another important performance measure is \emph{regret}.  In this section, we consider the $L$-constrained dynamic regret.  The motivation for our study is \cite{daniely2019competitive}, which provides an impossibility result showing that no algorithm can simultaneously maintain a constant competitive ratio and a sub-linear regret in the general setting of SOCO.  However, \cite{daniely2019competitive} utilizes linear hitting costs in its construction and thus it is an open question as to whether this impossibility result holds for strongly convex hitting costs. In this section, we show that the impossibility result does not hold for strongly convex hitting costs. % In particular, R-OBD can simultaneously provide both a constant competitive ratio and sublinear regret for the set of $m$-strongly convex hitting costs and Bergman Divergence movement costs.
To show this, we first characterize the parameters for which R-OBD gives sublinear  regret.
%\yiheng{Adam, are you sure this optimal result also holds for quadratic movement cost?} \adam{good catch!}

\begin{theorem}\label{R-OBD-RegretT1}
Consider hitting costs that are $m-$strongly convex with respect to a norm $\norm{\cdot}$ and movement costs defined as $c(x_t, x_{t-1}) = D_h(x_t || x_{t-1})$, where $h$ is $\alpha$-strongly convex and $\beta$-strongly smooth with respect to the same norm. Additionally, assume $\{f_t\}, h$ and its Fenchel Conjugate $ h^*$ are differentiable. Further, suppose that $\norm{\nabla h(x)}_*$ is bounded above by $G < \infty$, the diameter of the feasible set $\mathcal{X}$ is bounded above by $D$, and $\nabla h(0) = 0$. Then, for $\lambda_1, \lambda_2$ such that $\lambda_1 \geq 1 - \frac{m}{4\beta}$ and $\lambda_2 = \eta(T, L, D, G)$, where $\eta(T, L, D, G)$ is such that $\lim_{T\to \infty}\eta(T, L, D, G)\cdot \frac{D^2}{G}\sqrt{\frac{T}{L}} < \infty$, the $L$-constrained regret of R-OBD is $O(G\sqrt{TL})$. 
\end{theorem}

Theorem \ref{R-OBD-RegretT1} highlights that $O(G\sqrt{TL})$ regret can be achieved when $\lambda_1 \geq 1 - \frac{m}{4\beta}$ and $\lambda_2 \leq \frac{K G}{D^2}\cdot \sqrt{\frac{L}{T}}$ for some constant $K$. This suggests that the tendency to aggressively move towards the minimizer should shrink over time in order to achieve a small regret.  
%Setting $\lambda_2 = 0$ then obtains an algorithm that simultaneously achieves sublinear regret and a constant, dimension-free competitive ratio without requiring advance knowledge of the time horizon $T$.  In particular, combining Theorem \ref{R-OBD-RegretT1} with Theorem \ref{MainT1}, we see that to simultaneously achieve the optimal competitive ratio and $O(\sqrt{TL})$ regret we need to have
%$\lambda_1 = 2\big( 1 + \sqrt{1 + \frac{4\beta^2}{\alpha m}}\Big)^{-1}$ when $\lambda_2=0$. Therefore we know that when $m \geq \frac{4}{3}\frac{\beta^2}{\alpha}$, the inequality
%$\lambda_1 \geq 1 - \frac{m}{4\beta}$ holds. 
It is not possible to use Theorem \ref{R-OBD-RegretT1} to simultaneously achieve the optimal competitive ratio and $O(G\sqrt{TL})$ regret for all strongly convex hitting costs ($m>0$).  However, the corollary below shows that it is possible to simultaneously achieve a dimension-free, constant competitive ratio and an $O(G\sqrt{TL})$ regret for all $m>0$. An interesting open question that remains is whether it is possible to develop an algorithm that has sublinear regret and matches the optimal order for competitive ratio.

%\begin{corollary}
%Under the same conditions as Theorem \ref{R-OBD-RegretT1}, for all $m > 0$, R-OBD with parameters $\lambda_1 = 1, \lambda_2 = 0$ is $\left(1 + \frac{\beta^2}{\alpha m}\right)$-competitive and has a $O(G\sqrt{TL})$ regret simultaneously.
%\end{corollary}

%\yiheng{A more precise version is as follows:}
\begin{corollary} \label{cor:regretandCR}
Consider the same conditions as in Theorem \ref{R-OBD-RegretT1} and fix $m > 0$. R-OBD with parameters $\lambda_1 = \max \left(2\left( 1 + \sqrt{1 + \frac{4\beta^2}{\alpha m}}\right)^{-1}, 1 - \frac{m}{4\beta} \right), \lambda_2 = 0$ has an $O(G\sqrt{TL})$ regret and is $\max\left( \frac{1}{2}\left(1 + \sqrt{1 + \frac{4\beta^2}{\alpha m}}\right), 1 - \frac{\beta}{4\alpha} + \frac{\beta^2}{\alpha m}\right)$-competitive.
\end{corollary}

%%As $m\to 0^+$, Corollary \ref{cor:regretandCR} gives a competitive ratio of $\left(1 + \frac{\beta^2}{\alpha m}\right)$ for R-OBD, which is not on the same order as the lower bound in Theorem \ref{GeneralLowerT1}.  However, it is still a constant independent of dimension and thus Corollary \ref{cor:regretandCR} represents the first result showing that an algorithm can have sublinear regret and a constant competitive ratio for all $m$-strongly convex cost functions.  An interesting open question is to develop an algorithm that has sublinear regret and matches the optimal order for competitive ratio, or to show that no such algorithm exists.

%Let $\delta$ be the solution of $2\left( 1 + \sqrt{1 + \frac{4\beta^2}{\alpha m}}\right)^{-1} = 1 - \frac{m}{4\beta}$. When $m \geq \delta$, we set $\lambda_1 = 2\left( 1 + \sqrt{1 + \frac{4\beta^2}{\alpha m}}\right)^{-1}$ and the competitive ratio is $\frac{1}{2}\left(1 + \sqrt{1 + \frac{4\beta^2}{\alpha m}}\right)$; when $m < \delta$, we need to set $\lambda_1 = 1 - \frac{m}{4\beta}$ in order to pertain the $O(\sqrt{T})$ regret, the competitive ratio is $1 - \frac{\beta}{4\alpha} + \frac{\beta^2}{\alpha m}$. And we know $\delta < \frac{4\beta^2}{3\alpha}$.

%% file: Appendix.tex
%\section{Preliminaries}
\input{ConvexPrelim.tex}
\section{Proof of Theorem \ref{GeneralLowerT1}}
\input{GeneralLowerT1.tex}

\section{Proof of Theorem \ref{OBDLowerT1}}

\input{OBDLowerT1.tex}
\section{Proof of Theorem \ref{OBD_PLUS_T1}}
\input{OBD_PLUS_T1.tex}
\section{Proof of Theorem \ref{MainT1}}
\input{MainT1.tex}
\section{R-OBD with Squared $\ell_2$ Norm}
\input{thm-opt-cr.tex}
\section{Proof of Theorem \ref{WG_tradeT1}}
\input{WG_tradeT1.tex}
\section{Proof of Theorem \ref{R-OBD-RegretT1}}
\input{R-OBD-RegretT1.tex}

%% file: ConvexPrelim.tex
% To prove Theorem \ref{MainT1}, we need two properties of $\beta-$strongly smooth with respect to some general norm $\norm{\cdot}$ in Theorem \ref{DualNormT1}.

The appendices that follow provide the proofs of the results in the body of the paper.  Throughout the proofs in the appendix we use the following notation to denote the hitting and movement costs of the online learner: $H_t := f_t(x_t)$ and $M_t := c(x_t, x_{t-1})$, where $x_t$ is the point chosen by the online algorithm at time $t$.
Similarly, we denote the hitting and movement costs of the offline optimal (adversary) as $H_t^* := f_t(x_t^*)$ and $M_t^* := c(x_t^*, x_{t-1}^*)$, where $x_t^*$ is the point chosen by the offline optimal at time $t$.

Before moving to the proofs, we summarize a few standard definitions that are used throughout the paper. 

\begin{definition}
A function $f: \mathcal{X} \to \mathbb{R}$ is $\alpha$-strongly convex with respect to a norm $\norm{\cdot}$ if for all $x, y$ in the relative interior of the domain of $f$ and $\lambda \in (0, 1)$, we have
\[f(\lambda x + (1 - \lambda)y) \leq \lambda f(x) + (1 - \lambda)f(y) - \frac{\alpha}{2}\lambda (1 - \lambda)\norm{x - y}^2.\]
\end{definition}

\begin{definition}
A function $f: \mathcal{X} \to \mathbb{R}$ is $\beta$-strongly smooth with respect to a norm $\norm{\cdot}$ if $f$ is everywhere differentiable and if for all $x, y$ we have
\[f(y) \leq f(x) + \langle \nabla f(x), y - x \rangle + \frac{\beta}{2}\norm{y - x}^2.\]
\end{definition}

\begin{definition}
A function $f: \mathbb{R}^d \to \mathbb{R}$ is quasiconvex if its domain $\mathcal{X}$ and all its sublevel sets
\[ S_\alpha = \{x \in \mathcal{X} \mid f(x) \leq \alpha\},\]
for $\alpha \in \mathbb{R}$, is convex.
\end{definition}

\begin{definition}
For a norm $\norm{\cdot}$ in $\mathcal{X}$, its dual norm (on $\mathcal{X}$) $\norm{\cdot}_*$ is defined to be
\[ \norm{y}_* = \sup\{\langle x, y \rangle \mid \norm{x} \le 1 \}. \]
\end{definition}

\begin{definition}
For a convex function $f : \mathcal{X} \to \mathbb{R}$, its Fenchel Conjugate $f^*$ is defined to be
\[ f^*(y) = \sup\{\langle x, y \rangle - f(x) \mid x \in \mathcal{X} \}. \]
\end{definition}

Next, we introduce a few technical lemmas that are important throughout our analysis.

\input{DualNormT1.tex} 

Finally, before moving the the proofs of our main results, we prove two properties of the Bregman Divergence that play an important role in the analysis.

\input{GWGL1.tex}

\input{GWGL2.tex}

%% file: DualNormT1.tex
The first technical lemma is a characterization of strongly convex functions.

\begin{lemma}\label{DualNormT1}
 Suppose f is $\alpha-$strongly convex for some $\alpha > 0$ with respect to some norm $\norm{\cdot}$ and both $f$ and $f^*$ are differentiable, then $\forall \beta > 0$, the first condition implies the second condition and the third condition:
\begin{enumerate}
    \item $\forall x, y, f(y) \leq f(x) + \langle \nabla f(x), y - x\rangle + \frac{\beta}{2}\norm{x - y}^2$;
    \item $\forall x, y, f(y) \geq f(x) + \langle \nabla f(x), y - x\rangle + \frac{1}{2\beta}\norm{\nabla f(x) - \nabla f(y)}_*^2$;
    \item $\forall x, y, \norm{\nabla f(x) - \nabla f(y)}_* \leq \beta \norm{x - y}$.
\end{enumerate}
\end{lemma}

To prove Lemma \ref{DualNormT1}, we use Lemma \ref{ConjugateT1}, Lemma \ref{ConjugateL1}, and Lemma \ref{ConjugateL2} below.

The following lemma is Theorem 6 in \cite{kakade2009duality}.
\begin{lemma}\label{ConjugateT1}

If $f$ is convex and closed, the following two conditions are equivalent:
\begin{enumerate}
    \item $\forall x, y, f(y) \geq f(x) + \langle \nabla f(x), y - x\rangle + \frac{\beta}{2}\norm{x - y}^2$;
    \item $\forall x, y, f^*(y) \leq f^*(x) + \langle \nabla f^*(x), y - x\rangle + \frac{1}{2\beta}\norm{x - y}_*^2$
\end{enumerate}
i.e. f is $\beta-$strongly convex w.r.t some norm $\norm{\cdot}$ if and only if $f^*$ is $\frac{1}{\beta}$-strongly smooth w.r.t the dual norm $\norm{\cdot}_*$.
\end{lemma}

The next lemma is a special case of Lemma 17 in \cite{shalev2010equivalence}.
\begin{lemma}\label{ConjugateL1}
Let f be a closed, convex, and differentiable function. Then we have
\[ f^*(\nabla f(x)) + f(x) = \langle \nabla f(x), x\rangle. \]
\end{lemma}

\input{ConjugateL2.tex}

Using the three lemmas above, we now prove Lemma \ref{DualNormT1}.

\begin{proof}[Proof of Lemma \ref{DualNormT1}]
By the first condition and Lemma \ref{ConjugateT1}, we know $f^*$ is $\frac{1}{\beta}-$strongly convex with respect to $\norm{\cdot}_*$. Therefore we see
\[ f^*(\nabla f(y)) \geq f^*(\nabla f(x)) + \langle \nabla f^*(\nabla f(x)), \nabla f(y) - \nabla f(x)\rangle + \frac{1}{2\beta}\norm{\nabla f(x) - \nabla f(y)}_*^2. \]
Using Lemma \ref{ConjugateL1} and Lemma \ref{ConjugateL2}, we obtain
\[ \langle y, \nabla f(y)\rangle - f(y) \geq \left(\langle x, \nabla f(x)\rangle - f(x)\right) + \langle x, \nabla f(y) - \nabla f(x)\rangle + \frac{1}{2\beta}\norm{\nabla f(x) - \nabla f(y)}_*^2. \]
Rearranging the terms, we get
\[ f(x) \geq f(y) + \langle x - y, \nabla f(y)\rangle + \frac{1}{2\beta}\norm{\nabla f(x) - \nabla f(y)}_*^2, \]
which is the second condition.

The third condition follows from subtracting the second condition from the first condition.
\end{proof}

%% file: ConjugateL2.tex
\begin{comment}
%Please find the new proof below.
\begin{proof}[Proof of Lemma \ref{ConjugateL2}]
Let $\omega = (\nabla f)^{-1}(y)$. By lemma \ref{ConjugateL1}, we know
$$f^*(y) = \langle \omega, \nabla f(\omega)\rangle - f(\omega) = \langle (\nabla f)^{-1}(y), y\rangle - f((\nabla f)^{-1}(y))$$
For simplicity, let $h := (\nabla f)^{-1}$. Then
$$f^*(y) = \langle h(y), y\rangle - f(h(y))$$
Thus, let $c = h(y)$, we have
\begin{equation}
    \begin{aligned}
    \nabla f^*(y) &= (\frac{\partial f^*(y)}{\partial y})^T\\
    &= \Big(y^T \cdot \frac{\partial h}{\partial y}(y) + h(y)^T - \frac{\partial f}{\partial c}(c)\cdot \frac{\partial h}{\partial y}(y)\Big)^T\\
    &= \Big(y^T \cdot \frac{\partial h}{\partial y}(y) + h(y)^T - \nabla f(h(y))^T\cdot \frac{\partial h}{\partial y}(y)\Big)^T\\
    &= \Big(y^T \cdot \frac{\partial h}{\partial y}(y) + h(y)^T - y^T\cdot \frac{\partial h}{\partial y}(y)\Big)^T\\
    &= h(y)\\
    &= (\nabla f)^{-1} (y)
    \end{aligned}
\end{equation}
\end{proof}

\end{comment}

The following technical result describes a well-known property of the gradient of the Fenchel Conjugate.
\begin{lemma}\label{ConjugateL2}
Suppose f is $\alpha-$strongly convex for some $\alpha > 0$ with respect to some norm $\norm{\cdot}$ and both $f$ and $f^*$ are differentiable. Then we have
\[ x = \nabla f^*\left(\nabla f(x)\right), \forall x. \]
\end{lemma}

\begin{proof}%[Proof of Lemma \ref{ConjugateL2}]
For convenience, we define $y = \nabla f(x)$ and $x' = \nabla f^*(y)$. It suffices to prove that $x' = x$.

By Lemma \ref{ConjugateL1}, we obtain
\begin{equation}\label{ConjugateL2E1}
    f^*(y) + f(x) = \langle y, x\rangle = \langle x, y \rangle.
\end{equation}
Again by Lemma \ref{ConjugateL1}, we have

\begin{equation}\label{ConjugateL2E2}
    f(x') + f^*(y) = f^{**}(x') + f^*(y) = \langle x', y\rangle,
\end{equation}
where we use the fact that $f^{**} = f$.

Combining equations \eqref{ConjugateL2E1} and \eqref{ConjugateL2E2}, we obtain
\begin{equation*}
    0 = f(x) - f(x') - \langle x - x', y \rangle = f(x) + \langle x' - x, \nabla f(x)\rangle - f(x') \leq -\frac{\alpha}{2}\norm{x - x'}^2,
\end{equation*}
where in the last inequality we use the definition of $\alpha-$strongly convex. Therefore we have proved that $x = x'$.
\end{proof}

%% file: GWGL1.tex
%\textcolor{blue}{Fixed. Ready for reviewing.}
\begin{lemma}\label{GWGL1}
$\forall a, b, c \in \mathbb{R}^d$ and potential h, we have
\[ \langle \nabla h(b) - \nabla h(c), c - a\rangle = D_h(a || b) - D_h(a || c) - D_h(c || b). \]
\end{lemma}
\begin{proof}%[Proof of Lemma \ref{GWGL1}]
By the definition of Bregman Divergence, we obtain
\begin{equation*}
    \begin{aligned}
    D_h(a || b) - & D_h(a || c) - D_h(c || b)\\
    =\phantom{\,}& \left(h(a) - h(b) - \langle \nabla h(b), a - b \rangle\right) - \left(h(a) - h(c) - \langle \nabla h(c), a - c\rangle \right)\\
    &- \left(h(c) - h(b) - \langle \nabla h(b), c - b\rangle \right)\\
    =\phantom{\,}& -\langle \nabla h(b), a - b \rangle + \langle \nabla h(c), a - c\rangle + \langle \nabla h(b), c - b\rangle \\
    =\phantom{\,}& \left(-\langle \nabla h(b), a - b \rangle + \langle \nabla h(b), c - b\rangle \right) + \langle \nabla h(c), a - c\rangle \\
    =\phantom{\,}& \langle \nabla h(b), c - a\rangle + \langle \nabla h(c), a - c\rangle \\
    =\phantom{\,}& \langle \nabla h(b) - \nabla h(c), c - a \rangle.
    \end{aligned}
\end{equation*}
%\gautam{Pull Bregman lemmas into separate section?}
\end{proof}

%% file: GWGL2.tex
%\textcolor{blue}{Fixed. Ready for reviewing.}

\begin{lemma}\label{GWGL2}
For all $a, b, c\in \mathbb{R}^d$, we have
$$D_h(c || a) - D_h(c || b) = D_h(0 || a) - D_h(0 || b) + \langle \nabla h(b) - \nabla h(a), c\rangle.$$
\end{lemma}
\begin{proof}%[Proof of Lemma \ref{GWGL2}]
Using the definition of Bregman divergence, we obtain
\begin{equation*}
    \begin{aligned}
    D_h(c || a) -& D_h(c || b)\\
    =\phantom{\,}& h(c) - h(a) - \langle \nabla h(a), c - a\rangle - h(c) + h(b) + \langle \nabla h(b), c - b\rangle\\
    =\phantom{\,}& \big( h(b) - \langle \nabla h(b), b\rangle \big) - \big( h(a) - \langle \nabla h(a), a\rangle \big) + \langle \nabla h(b) - \nabla h(a), c\rangle\\
    =\phantom{\,}& D_h(0 || a) - D_h(0 || b) + \langle \nabla h(b) - \nabla h(a), c\rangle.
    \end{aligned}
\end{equation*}
\end{proof}

%% file: GeneralLowerT1.tex
We consider a sequence of hitting cost functions on the real line such that the algorithm stays at the starting point through time steps $t = 1, 2, \cdots, n$ and is forced to incur a huge movement cost at time step $t = n + 1$, whereas the offline adversary can pay relatively little cost by dividing the long trek between $x_0$ and $v_{n+1}$ into multiple small steps through time steps $t = 1, 2, \cdots, n+1$.

Specifically, suppose the starting point of the algorithm and the offline adversary is $x_0 = x_0^* = 0$, and the hitting cost functions are
\[ f_t(x) = \begin{cases}
\frac{m}{2}x^2 & t\in \{1, 2, \cdots, n\}\\
\frac{m'}{2}(x - 1)^2 & t = n+1
\end{cases} \]
for some large parameter $m'$ that we choose later.

Suppose the algorithm first moves at time step $t_0$. If $t_0 < n+1$, we stop the game at time step $t_0$ and compare the algorithm with an offline adversary which always stays at $x = 0$. The total cost of offline adversary is 0, but the total cost of the algorithm is non-zero. So, the competitive ratio is unbounded.

Next we consider the case where $t_0 \geq n+1$. This implies that $x_1, \ldots x_n = 0$ and $x_{n+1}$ is some non-zero point, say $x$. We see that the cost incurred by the online algorithm is
\[ cost(ALG) \geq \min_{x_{n+1}}(M_{n+1} + H_{n+1}) = \min_x \left(\frac{1}{2}x^2 + \frac{m'}{2}(x - 1)^2\right).\]
Notice that the right hand side tends to $\frac{1}{2}$ as $m'$ tends to infinity; specifically, we have 
\begin{equation}\label{ALG}
    cost(ALG) \geq \min_x \left(\frac{1}{2}x^2 + \frac{m'}{2}(x - 1)^2\right) = \frac{1}{2\left(1 + \frac{1}{m'}\right)}.
\end{equation}

Now let us consider the offline optimal.  Notice that, in the limit as $m'$ tends to infinity, the offline optimal must satisfy $x_0^*=0$ and $x_{n+1}^* = 1$; otherwise it would incur unbounded cost. Our lower bound is derived by considering the case when $m'\to\infty$ and so we constrain the adversary to satisfy the above, knowing that the adversary is not optimal for finite $m'$, i.e., $cost(ADV)\geq cost(OPT)$ with $cost(ADV)\to cost(OPT)$ as $m'\to\infty$.

Let the sequence of points the adversary chooses as $x^* = (x_0^*, x_1^*, \cdots, x_{n+1}^*) \in \mathbb{R}^{n+2}$. We compute the cost incurred by the adversary as follows where, to simplify presentation, we define $\mathcal{K}(n, y)$ to be the set $\{x \in \mathbb{R}^{n+2} \mid x_i \leq x_{i+1}, x_0 = 0, x_{n+1} = y\}$. 
\begin{equation}
    \begin{aligned}
    a_n &= 2\min_{x^* \in \mathcal{K}(n, 1)}\sum_{i=1}^{n+1}(H_i^* + M_i^*)\\
    &= 2\min_{x^* \in \mathcal{K}(n, 1)}\left(\sum_{i=1}^n \frac{m}{2}(x_i^*)^2 + \sum_{i=1}^{n+1}\frac{1}{2}(x_i^* - x_{i-1}^*)^2\right).
    \end{aligned}
    \nonumber
\end{equation}
In words, $a_n$ is twice the minimal offline cost subject to the constraints $x_0^* = 0, x_{n+1}^* = 1$. We derive the limiting behavior of the offline costs as $n \to \infty$ in the following lemma. 

\begin{lemma}\label{GeneralLowerL1}
For $m > 0$, define
\[ a_n = 2\min_{x^* \in \mathcal{K}(n, 1) }\left(\sum_{i=1}^n \frac{m}{2}(x_i^*)^2 + \sum_{i=1}^{n+1}\frac{1}{2}(x_i^* - x_{i-1}^*)^2\right).
\]
Then we have $\lim_{n \to \infty} a_n = \frac{-m + \sqrt{m^2 + 4m}}{2}$.
\end{lemma}

Given the lemma, the total cost of the offline adversary will be $\frac{a_n}{2}$. Finally, applying \eqref{ALG}, we know $\forall n$ and $\forall m' > 0$, 
\[ \frac{cost(ALG)}{cost(ADV)} \geq \frac{\frac{1}{2(1 + \frac{1}{m'})}}{\frac{a_n}{2}} = \frac{1}{(1 + \frac{1}{m'})a_n}.\]
By taking the limit $n \to \infty$ and $m' \to \infty$ and using Lemma \ref{GeneralLowerL1}, we obtain
\[ \frac{cost(ALG)}{cost(OPT)} = \lim_{n,m'\to\infty}\frac{cost(ALG)}{cost(ADV)} \geq \left(\frac{-m + \sqrt{m^2 + 4m}}{2}\right)^{-1} = \frac{1 + \sqrt{1 + \frac{4}{m}}}{2}.\]

\begin{comment}
\subsubsection*{An Alternative Method to compute the limit of $\{a_n\}$}
(Since Haoyuan suggested that the first proof for Lemma \ref{GeneralLowerL1} is too complicated, an alternative is shown below. Do you think it is simpler?)
\begin{proof}
Now we use induction to prove that for all n, we have
\begin{equation}\label{GeneralLower_induction}
    a_n \geq x_1.
\end{equation}
Notice that $a_0 = 1 \geq x_1$.

Suppose inequality \ref{GeneralLower_induction} holds for $n = k$. Then we have
\[ a_{k+1} = \frac{a_k + m}{a_k + m + 1} \geq \frac{x_1 + m}{x_1 + m + 1} = x_1. \]
So inequality \ref{GeneralLower_induction} also holds for $n = k+1$.

By induction, we know inequality \ref{GeneralLower_induction} holds for all n.

Using inequality \ref{GeneralLower_induction}, we obtain
\[ - a_n^2 - m a_n + m \leq 0.\]
Thus $\forall n$, we have
\[a_{n+1} - a_n = \frac{a_n + m}{a_n + m + 1} - a_n = \frac{- a_n^2 - m a_n + m}{a_n + m + 1} \leq 0.\]
Therefore we obtain
\begin{equation}\label{GeneralLower_Decrease}
    a_{n+1} \leq a_n, \forall n.
\end{equation}

Combining inequalities \ref{GeneralLower_induction} and \ref{GeneralLower_Decrease}, we know $\lim_{n\to \infty}a_n$ exists. Suppose $L = \lim_{n\to \infty}a_n$.

Solving the equation
\[L = \frac{L + m}{L + m + 1}\]
and use the knowledge that $L \geq x_1$, we get $L = x_1$, which means
\[\lim_{n\to \infty}a_n = x_1 = \frac{-m + \sqrt{m^2 + 4m}}{2}.\]
\end{proof}

\end{comment}

All that remains is to prove Lemma \ref{GeneralLowerL1}, which describes the cost of the offline adversary in the limit as $n$ tends to infinity.

\begin{proof}[Proof of Lemma \ref{GeneralLowerL1}]

Using the fact that the costs are all homogeneous of degree 2, we see that for all $y \in [0, 1]$, we have
\begin{equation}\label{GeneralLowerE1}
\begin{aligned}
&\min_{x^* \in \mathcal{K}(n, y)}\left(\sum_{i=1}^n \frac{m}{2}(x_i^*)^2 + \sum_{i=1}^{n+1}\frac{1}{2}(x_i^* - x_{i-1}^*)^2\right)\\
={}& y^2\min_{x^* \in \mathcal{K}(n, 1)}\left(\sum_{i=1}^n \frac{m}{2}(x_i^*)^2 + \sum_{i=1}^{n+1}\frac{1}{2}(x_i^* - x_{i-1}^*)^2\right).
\end{aligned}
\end{equation}

The sequence $\{a_n\}, n \geq 0$ has a recursive relationship as follows:
\begin{equation}
    \begin{aligned}
    a_{n+1} &= 2\min_{x^* \in \mathcal{K}(n+1, 1)}\left(\sum_{i=1}^{n+1} \frac{m}{2}(x_i^*)^2 + \sum_{i=1}^{n+2}\frac{1}{2}(x_i^* - x_{i-1}^*)^2\right)\\
    &= 2\min_{0 \leq x\leq 1}\Bigg( \min_{x^* \in \mathcal{K}(n, x)}\Big(\sum_{i=1}^n \frac{m}{2}(x_i^*)^2 + \sum_{i=1}^{n+1}\frac{1}{2}(x_i^* - x_{i-1}^*)^2\Big)\\
    &\quad + \frac{m}{2}x^2 + \frac{1}{2}(1 - x)^2\Bigg)\\
    &= 2\min_{0 \leq x\leq 1}\Bigg( x^2\min_{x^* \in \mathcal{K}(n, 1)}\Big(\sum_{i=1}^n \frac{m}{2}(x_i^*)^2 + \sum_{i=1}^{n+1}\frac{1}{2}(x_i^* - x_{i-1}^*)^2\Big)\\
    &\quad + \frac{m}{2}x^2 + \frac{1}{2}(1 - x)^2\Bigg)\\
    &= 2\min_{0\leq x \leq 1}\left( \frac{a_n}{2}x^2 + \frac{m}{2}x^2 + \frac{1}{2}(1 - x)^2\right)\\
    &= \frac{a_n + m}{a_n + m + 1}.
    \end{aligned}
\end{equation}
%Here we applied the homogeneity observation to express the minimization via dynamic programming.

Solving the equation $x = \frac{x + m}{x + m + 1}$, we find the two fixed points of the recursive relationship $a_{n+1} = \frac{a_n + m}{a_n + m + 1}$ are
\[ x_1 = \frac{-m + \sqrt{m^2 + 4m}}{2}, \]
and
\[ x_2 = \frac{-m - \sqrt{m^2 + 4m}}{2}.\]

Notice that for $i = 1, 2$, we have
\[m - (m + 1)x_i = -(1 - x_i)x_i.\]
Using this property, we obtain
\begin{equation}\label{GeneralLowerE2}
    a_{n+1} - x_1 = \frac{a_n + m}{a_n + m + 1} - x_1 = \frac{(1 - x_1)a_n + m - (m + 1)x_1}{a_n + m + 1} = \frac{(1 - x_1)(a_n - x_1)}{a_n + m + 1},
\end{equation}
and
\begin{equation}\label{GeneralLowerE3}
    a_{n+1} - x_2 = \frac{a_n + m}{a_n + m + 1} - x_2 = \frac{(1 - x_2)a_n + m - (m + 1)x_2}{a_n + m + 1} = \frac{(1 - x_2)(a_n - x_2)}{a_n + m + 1}.
\end{equation}
Notice that $a_{n+1} - x_2 > 0$. By dividing equations \eqref{GeneralLowerE2} and \eqref{GeneralLowerE3}, we obtain
$$\left(\frac{a_{n+1} - x_1}{a_{n+1} - x_2}\right) = \frac{1 - x_1}{1 - x_2}\cdot \left(\frac{a_{n} - x_1}{a_{n} - x_2}\right), \forall n\geq 0.$$
Remember that $a_0 = 1$. Therefore we have
$$\left(\frac{a_{n} - x_1}{a_{n} - x_2}\right) = \left(\frac{1 - x_1}{1 - x_2}\right)^n \left(\frac{a_{0} - x_1}{a_{0} - x_2}\right) = \left(\frac{1 - x_1}{1 - x_2}\right)^{n+1}.$$
Rearranging this equation, we get
$$a_n = \left(1 - \left(\frac{1 - x_1}{1 - x_2}\right)^{n+1} \right)^{-1}\left(x_1 - x_2 \cdot \left(\frac{1 - x_1}{1 - x_2}\right)^{n+1}\right).$$
Since $0 < \left(\frac{1 - x_1}{1 - x_2}\right) < 1$, we have
\begin{equation}\label{ADVLim}
    \lim_{n\to \infty}a_n = x_1 = \frac{-m + \sqrt{m^2 + 4m}}{2}.
\end{equation}
\end{proof}

%% file: OBDLowerT1.tex
\input{OBDLowerL1.tex}

Now we move on to the next technical lemma in the proof of Theorem \ref{OBDLowerT1}.

\begin{lemma}\label{OBDLowerL2}

When $\gamma = o(\frac{1}{m})$, the competitive ratio of OBD is $\Omega(\sqrt{\frac{\gamma}{m}})$.
\end{lemma}
\begin{proof}%[Proof of Lemma \ref{OBDLowerL2}]
We consider a sequence of cost functions on the real line such that the OBD algorithm moves far away from the starting point, incurring significant movement costs, whereas the offline adversary could pay relatively little cost by staying at the starting point. More specifically, we consider the sequence of hitting cost functions $f_t(x) = \frac{m}{2}(x-t)^2, t = 1, 2, \cdots, n$. The value of $n$ will be picked later. We assume the starting point is at zero.

\begin{figure}[t]
    \begin{center}
        \includegraphics{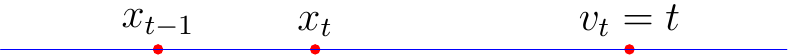}
        \caption{\emph{Balance condition at time step $t$ in Lemma \ref{OBDLowerL2}. Starting from $x_{t-1}$, OBD picks $x_t$ after observing the hitting cost function $f_t(x) = \frac{m}{2}(x - t)^2$.}}
        \label{figure:OBDLowerL2}
    \end{center}
\end{figure}

Notice that by the balance condition we always have $M_t = \gamma H_t$, so $\frac{1}{2}\|x_t - x_{t-1}\|^2 = \gamma \frac{m}{2}\|x_t - t\|^2.$ We can rearrange this expression to obtain $\frac{x_t - x_{t-1}}{t - x_{t}} = \sqrt{\gamma m}$. Define
\[ \lambda = \frac{x_t - x_{t-1}}{t - x_{t-1}} = \frac{\sqrt{\gamma m}}{1 + \sqrt{\gamma m}}.\]
We obtain the recursive equation $x_t = x_{t-1} + (t - x_{t-1})\lambda$ with initial condition $x_0 = 0$. Solving this equation, we obtain $x_t = t - \frac{1 - \lambda}{\lambda}(1 - (1 - \lambda)^t)$.

Suppose we picked $n$ to be = $\lceil \frac{1}{\lambda} \rceil $. By assumption, $\gamma = o(\frac{1}{m})$; hence in the limit as $m$ tends to zero, $\lambda$ also tends to zero.  Notice that $x_n = n - \frac{1 - \lambda}{\lambda}(1 - (1 - \lambda)^n) \geq \frac{1}{\lambda}\frac{1}{2e} - (1 - \frac{1}{e}) \geq \frac{1}{6\lambda}$ for sufficiently small $\lambda$. Here we used the fact that $(1 - \lambda)^\frac{1}{\lambda} \rightarrow e^{-1}$.

Suppose the next cost function is $f_{n+1}(x) = m'x^2$. Notice that if the offline adversary simply stays at zero throughout the game, the total cost it incurs would be
\[ cost(ADV) = \frac{m}{2}(1^2 + 2^2 + \cdots + n^2)\leq \frac{mn^3}{2} = \Theta \left(\frac{m}{\lambda^3} \right) = \Theta \left (\frac{1}{\sqrt{\gamma^3 m}} \right).\]
In the last step, we used the fact that $\lambda$ tends to $\sqrt{\gamma m}$ when $\gamma = o(\frac{1}{m})$ and $m$ tends to zero.

If we pick $m'$ large enough that OBD is forced to incur movement cost at least $\frac{1}{2}(\frac{x_n}{2})^2$, the total cost incurred by OBD is
\[cost(OBD) \geq \frac{1}{2} \left(\frac{x_n}{2} \right)^2  = \Theta \left(\frac{1}{\lambda^2} \right) = \Theta \left(\frac{1}{\gamma m} \right).\]
Putting these facts together, we see that the competitive ratio is at least $\Theta(\sqrt{\frac{\gamma}{m}})$.
\end{proof}

The last technical lemma used to proof Theorem \ref{OBDLowerT1} is the following.

\begin{lemma}\label{OBDLowerL3}
When $\gamma = \Omega(\frac{1}{m})$, the competitive ratio of OBD is $\Omega\left(\frac{1}{m}\right)$.
\end{lemma}
\begin{proof}%[Proof of Lemma \ref{OBDLowerL3}]
Since $\gamma = \Omega(\frac{1}{m})$, we can assume there exists $C > 0$ such that $\gamma \geq C/m$. We again consider a situation such that the OBD algorithm moves far away from the starting point, incurring significant movement cost, whereas the offline adversary could pay relatively little cost by staying at the starting point. More specifically, suppose the starting point is zero and the first cost function is $f_1(x) = \frac{m}{2}(1 - x)^2$. Suppose the adversary stays at zero. The cost incurred by the adversary will be
\[cost(ADV) = \frac{m}{2}.\]
Notice that by the balance condition ($M_t = \gamma H_t$), the point $x_1$ picked by OBD satisfies $\frac{x_1^2}{2} = \gamma \frac{m}{2}(1 - x_1)^2$. So the cost incurred by OBD is lower bounded by
\[cost(OBD) \geq M_1 =  \frac{1}{2}\left( \frac{\sqrt{\gamma m}}{1 + \sqrt{\gamma m}}\right)^2 \geq \frac{1}{2}\left( \frac{\sqrt{C}}{1 + \sqrt{C}}\right)^2.\]
Since $C$ is a positive constant, the competitive ratio of OBD is lower bounded by $\frac{OBD}{ADV} = \Theta \left( \frac{1}{m}\right).$
\end{proof}

Now we return to the proof of Theorem \ref{OBDLowerT1}.  This proof is a straightforward combination of the above lemmas. When $\gamma = o(\frac{1}{m})$, by combining Lemma \ref{OBDLowerL1} and Lemma \ref{OBDLowerL2}, we know the competitive ratio is at least $\max \left(\frac{C_1}{\gamma m}, C_2\sqrt{\frac{\gamma}{m}}\right)$ for some positive constants $C_1, C_2$. Notice that function $\frac{C_1}{\gamma m}$ is monotonically decreasing in $\gamma$ and $C_2\sqrt{\frac{\gamma}{m}}$ is monotonically increasing in $\gamma$. Solving the equation $\frac{C_1}{\gamma m} = C_2\sqrt{\frac{\gamma}{m}}$, we get $\gamma = \left(\frac{C_1}{C_2}\right)^{\frac{2}{3}}m^{-\frac{1}{3}}$. Therefore we see that
\[\max \left\{\frac{C_1}{\gamma m}, C_2\sqrt{\frac{\gamma}{m}}\right\} \geq C_1^{\frac{1}{3}}C_2^{\frac{2}{3}}m^{-\frac{2}{3}} = \Theta(m^{-\frac{2}{3}}).\]
On the other hand, when $\gamma = \Omega(\frac{1}{m})$, by Lemma \ref{OBDLowerL3}, we know the competitive ratio of OBD is lower bounded by $\Theta\left(\frac{1}{m}\right)$.

Together, the above implies that the competitive ratio of OBD is at least $\Theta(m^{-\frac{2}{3}})$ when $m \to 0^+$.

%\end{proof}

%% file: OBDLowerL1.tex
Our proof of Theorem \ref{OBDLowerT1} relies on a set of technical lemmas, which follow.   Lemma \ref{OBDLowerL1} and Lemma \ref{OBDLowerL2} work together to establish a lower bound on the competitive ratio as $m$ tends to zero when the balance parameter $\gamma$ is set to be $ o(1/m)$ , while Lemma \ref{OBDLowerL3} lower bound on the competitive ratio as $m$ tends to zero when the balance parameter $\gamma$ is set to be $\Omega(1/m)$.

\begin{lemma}\label{OBDLowerL1}
If $\gamma = o(1/m)$, the competitive ratio of OBD is $\Omega(1/(\gamma m))$ when $m\to 0^+$.
\end{lemma}

\begin{proof}%[Proof of Lemma \ref{OBDLowerL1}]

Our approach is to construct a scenario where OBD is forced to move along the circumference of a large circle, but the offline adversary moves along the circumference of a much smaller circle (see Figure \ref{figure:OBDLowerMainBody}). The adversary is hence able to pay much smaller movements costs, forcing the competitive ratio to be large. %We divide the analysis into two steps. We first show how to force the OBD algorithm onto the circle. We then exhibit a series of cost functions which force the OBD algorithm to perpetually move along the circle. 

 We propose a series of costs which force OBD to move in a circle. The idea is to construct a cost function so that, at the end of every round, the relative positions of the OBD algorithm, the offline adversary, and the minimizer are fixed. Since OBD is memoryless, we can simply input this function arbitrarily many times and the positions of OBD and the offline adversary will trace out a pair of concentric circles (see Figure \ref{figure:OBDLowerMainBody}).
 
 Suppose that, at the start of a round, OBD is at the point $A$. Let $\ell$ be the distance between OBD and the adversary.  Consider a right triangle $ABC$ such that $|AB| = h = \sqrt{\gamma m} \ell$, the offline adversary is at some point $D$ on the hypotenuse $AC$ and $|AD| = |BC|= \ell$ (see Figure \ref{figure:OBDLowerProof2}). Let us introduce a coordinate system such that the origin lies at $C$, the $x$-axis contains $BC$ and the $y$-axis is parallel to $AB$, such that the positive part of the axis lies on the same side of $BC$ as the segment $AC$. Our goal is to construct a cost function which forces OBD towards $B$. This will preserve the relative positions of OBD and the adversary, since we assumed that they were a distance $\ell$ away at the start of the round. Consider the costs $g(u) = \frac{m}{2}\norm{u - C}^2$, $h(u) = \alpha \cdot d(u, BC)$ where $d(u, BC)$ is the distance from the point $u$ to the line passing through $B$ and $C$ and $\alpha >0$ is a parameter we will pick later. Define $f_t(u) = h(u) + g(u)$. Notice that $f_t$ is $m$-strongly convex because it is the sum of an $m$-strongly convex function and a convex function. Intuitively, when $\alpha$ is large, the function $f_t$ is infinity outside of the line $BC$ but is equal to $g(u) = \frac{m}{2} \|u - C\|^2$ when restricted to points $u$ on the line. After observing the cost $f_t$, OBD will pick some new point $E$.
 
 The following lemma highlights that $E$ can be driven arbitrarily close to $B$ by taking $\alpha$ to be sufficiently large.
 \begin{lemma}\label{OBDLowerL1:Close}
 Let $\varepsilon > 0$, and suppose $\alpha$ is picked to that $\alpha > \frac{hm\ell^2}{\varepsilon^2}$. Then the point $E$ picked by OBD satisfies $|EB| < \varepsilon$.
 \end{lemma}

%\textcolor{blue}{Gautam: we are now calling the distance x to be z. }
\input{Figure_OBDLowerProof2.tex}
We instruct the adversary to pick the point $F$ on the line $BC$ (the $x$-axis) such that $EF = \ell$ (see Figure \ref{figure:OBDLowerProof2}). Notice that $|CF| = |BF| - |BC| \leq |BE| + |EF| - |BC| = |EB| + \ell - \ell < \varepsilon$, where we used the triangle inequality. Let $z = |DC|$. We see that the total cost incurred by the offline adversary is
\[ M_t^* + H_t^* = \frac{1}{2}|DF|^2 + \frac{m}{2}|CF|^2 \leq \frac{1}{2}(|DC|+|CF|)^2 + \frac{m}{2}|CF|^2 \leq \frac{1}{2}(z + \varepsilon)^2 + \frac{m\varepsilon^2}{2},\]
where we applied the triangle inequality. 

Notice that $h = |AB| = \sqrt{|AC|^2 - |BC|^2}$ by the Pythagorean theorem (recall that $ABC$ is a right triangle). Since $|AC| = \ell + z$ and $|BC| = \ell$, we see that $h = \sqrt{2z\ell + z^2}$.
Hence the movement cost incurred by the OBD is
\[ M_t \geq \frac{1}{2}(h - \varepsilon)^2 = \frac{1}{2}(\sqrt{2z\ell + z^2} - \varepsilon)^2.\]

Hence the ratio of the costs is 
\[ \frac{M_t + H_t}{M_t^* + H_t^*} \geq \frac{M_t}{M_t^* + H_t^*} \geq \frac{\frac{1}{2}(\sqrt{2z\ell + z^2} - \varepsilon)^2}{\frac{1}{2}(z + \varepsilon)^2 + \frac{m\varepsilon^2}{2}}.\]

Since the limit of this expression as $\varepsilon \to 0$ is $\frac{2z\ell + z^2}{z^2}$, for sufficiently small $\varepsilon$ this will be at least $\frac{1}{2}\frac{2z\ell + z^2}{z^2} \ge \frac{\ell}{z}$. Since $z = \sqrt{h^2 + \ell^2} - \ell$ and $h = \sqrt{\gamma m} \ell$, the ratio of costs is at least 
\[ \frac{\ell }{\sqrt{\gamma m \ell^2 + \ell^2} - \ell} = \frac{1 }{\sqrt{\gamma m + 1} - 1} = \frac{\sqrt{\gamma m + 1} + 1 }{\gamma m} \geq \frac{2}{\gamma m}.\]

Now, we describe the whole process. When $t = 1$, the hitting cost function is $f_1(x) = \frac{m}{2}\norm{x}_2^2$. While OBD stays at $x = 0$, the adversary moves to the point $(\ell, 0)$; it incurs a one-time cost of $M_1^* + H_1^* = \frac{1}{2}\ell^2 + \frac{m}{2}\ell^2$. On all subsequent steps $t = 2 \ldots T$, we repeatedly apply the construction, which forces OBD to move in a circle. The one-time cost incurred by the adversary to setup the game is negligible in the limit as $T$ is large, and the per-round ratio of costs is $\Omega(\frac{1}{\gamma m})$, so the competitive ratio is also $\Omega(\frac{1}{\gamma m})$ as claimed. 
\end{proof}

The key technical lemma used in the proof is Lemma \ref{OBDLowerL1:Close}, and we now provide a proof of that result.

\input{CloseLemma.tex}

%% file: Figure_OBDLowerProof2.tex
\begin{figure}
    \begin{center}
        \includegraphics{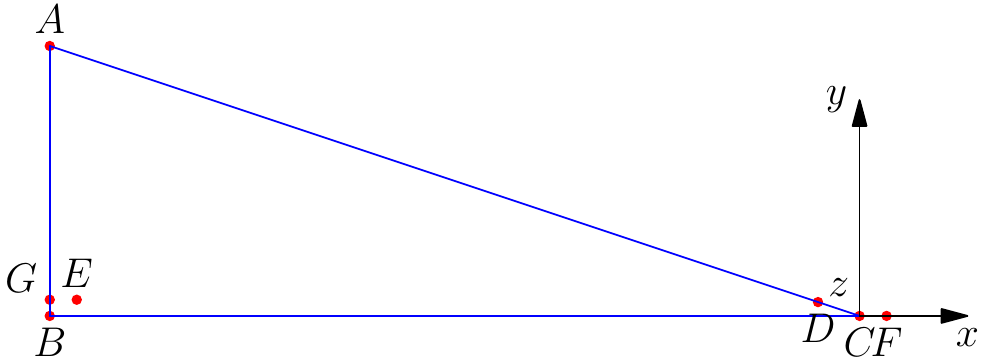}
    \end{center}
\caption{\emph{In the right triangle $\triangle ABC$, $\angle ABC = 90^o, |BC| = \ell, |AB| = h = \sqrt{\gamma m}\ell$. Point $D$ is on the line segment $AC$ such that $|AD| = \ell$. OBD starts at point $A$ and selects point $E$. The offline adversary starts at point $D$ and selects point $F$. $G$ is the projection point of $E$ on line segment $AB$.}}
\label{figure:OBDLowerProof2}
\end{figure}

%% file: CloseLemma.tex
\begin{proof}[Proof of Lemma \ref{OBDLowerL1:Close}]
 Suppose $\alpha > \frac{hm\ell^2}{\varepsilon^2}$. We first show that OBD selects the point $E$ strictly contained by the $\frac{m}{2}\ell^2$-level set, which is the one $B$ lies on. First observe that the point $B$ satisfies the balance condition: $\frac{1}{2} |AB|^2 = \gamma \frac{m}{2}|BC|^2$, because we constructed $ABC$ so that $|AB| = h = \sqrt{\gamma m} \ell$ and $|BC| = \ell$. However, the point $B$ is not necessarily a projection of $A$ onto any level set of $f_t$. If OBD projected onto the level set which $B$ lies on, it would incur less cost than if it moved to $B$; however then the balance condition would be violated. To restore the balance condition, we must increase the movement cost while decreasing the hitting cost -- which means we must move to a strictly smaller level set, say the $\frac{m}{2}l_1^2$-level set, where $l_1 < l$.

Let $E_y$ denote the $y$-coordinate of $E$, using the coordinate system we define in the proof of Lemma \ref{OBDLowerL1}. Notice that $E_y = \frac{g(E)}{\alpha}$, since $g(E)$ was defined to be the vertical distance to the $x$-axis times $\alpha$. Since $g(E) \leq f_t(E)$, we see that $E_y \leq \frac{f_t(E)}{\alpha} = \frac{ml_1^2}{2\alpha} \leq \frac{ml^2}{2\alpha}$, where we used the fact that $E$ lies on the $\frac{m}{2}\ell_1^2$ level set and $\ell_1 \leq \ell$. By the balance condition, $\frac{1}{2}|AE|^2 = \frac{\gamma m}{2}l_1^2 \leq \frac{\gamma m}{2}l^2 = \frac{1}{2}h^2$. Let $G$ be the point with coordinates $(B_x, E_y)$. Applying the Pythagorean theorem successively to the right triangle $BEG$ and the right triangle $AEG$, we see that
\begin{equation}
    \begin{aligned}
    |EB|^2 &= |E_x - B_x|^2 + E_y^2\leq (|AE|^2 - (|AB| - E_y)^2) + E_y^2\\
    &\leq (|AB|^2 - (|AB| - E_y)^2) + E_y^2 \leq 2h\cdot E_y \leq h\frac{ml^2}{\alpha},
    \end{aligned}
\end{equation}
where we used the fact that $|AB| \geq |AE|$ and $|AB| = h$. Since we picked $\alpha > \frac{hm\ell^2}{\varepsilon^2}$, we see that $|EB| < \varepsilon$.
 
 \end{proof}

%% file: OBD_PLUS_T1.tex
%\textcolor{blue}{Roughly fixed. Waiting for further tuning.}

To begin, note that it is sufficient to prove result for all positive $m \le \frac{9}{64}$. Similarly, it also suffices to show Theorem \ref{OBD_PLUS_T1} when the minimum of every hitting cost function is zero, since otherwise the competitive ratio can only improve if this is not the case.

Our argument makes use of the following potential function: $\phi(x_t, x_t^*) = \eta \norm{x_t - x_t^*}^2$.
We define $\Delta \phi = \phi(x_t, x_t^*) - \phi(x_{t-1}, x_{t-1}^*)$ and $\Delta \phi' = \phi(x_t', x_t^*) - \phi(x_{t-1}, x_{t-1}^*)$.
It suffices to show that $H_t + M_t + \Delta \phi \leq C(H_t^* + M_t^*)$, for some positive constant $C$.
From this inequality, we can sum over all timesteps $t$ to yield that the competitive ratio is upper bounded by $C$:
\[ \sum_{t=0}^T H_t + M_t \le \sum_{t=0}^T H_t + M_t + \Delta\phi \le C \sum_{t=0}^T\left( H^*_t + M^*_t\right).\]

Throughout the proof, we fix $\eta = 4$ and use $\norm{\cdot}$ to denote $\ell_2$ norm. When we refer to generalized mean inequality, we mean
\[ (a + b)^2 \leq 2a^2 + 2b^2, \forall a, b \in \mathbb{R}.\] 
We define $H_t' := f_t(x_t')$ and $M_t' := c(x_t', x_{t-1}) = \frac{1}{2}\norm{x_t' - x_{t-1}}_2^2$, where $x_t'$ is the point chosen by the first OBD phase (line 3) of Algorithm \ref{alg:OBD_PLUS}.

Before we move to the main casework in the proof, we begin with a technical lemma that we use to bound the change in the potential function.

\begin{lemma}
        \label{thm:quad}
        Suppose the potential function $\phi : \mathbb{R}^d \times \mathbb{R}^d \to \mathbb{R}_{\geq 0}$ is defined as $\phi(a, b) = \eta \norm{a - b}^2$, where $\eta > 0$. Then $\forall \, \lambda > 0$, the change in potential satisfies
        \[ \phi(a, c) - \phi(a, b)\leq (1+\lambda^2)\phi(b, c) + \frac{1}{\lambda^2}\phi(a, b),\]
        for all $a, b, c \in \mathbb{R}^d$.
\end{lemma}
\begin{proof}%[Proof of Lemma \ref{thm:quad}]
        Using the triangle inequality, we obtain
        \[ \norm{a - c}^2 \leq (\norm{a - b} + \norm{b - c})^2 = \norm{a - b}^2 + \norm{b - c}^2 + 2\norm{a - b}\norm{b - c}.\]
        Rearranging the terms, we obtain
        \begin{equation*}
            \begin{aligned}
                \norm{a - c}^2 - \norm{a - b}^2 
                &\leq \norm{b - c}^2 + 2\norm{a - b} \norm{b - c} \\
                &= \norm{b - c}^2 + 2(\frac{1}{\lambda}\norm{a - b}) (\lambda \norm{b - c}) \\
                &\leq (1 + \lambda ^ 2) \norm{b - c}^2 + \frac{1}{\lambda^2}\norm{a - b}^2,
            \end{aligned}
        \end{equation*}
        where in the last line we use the AM-GM inequality.
\end{proof}

We are now ready to precede with the proof, which is divided up into two cases based on the relationship between the hitting cost of the algorithm and that of the adversary.

\subsubsection*{Case 1: $H_t'\leq H_t^*$}
Since the hitting cost function satisfies $f_t(x) \geq \frac{m}{2}\norm{x - v_t}^2$, by the triangle inequality, we have
\begin{equation}\label{OBDPlusE1}
    \norm{x_t' - x_t^*} \leq \norm{x_t' - v_t} + \norm{x_t^* - v_t} \leq  \left(\sqrt{\frac{2H_t'}{m}} + \sqrt{\frac{2H_t^*}{m}}\right).
\end{equation}
Thus the change in potential satisfies
%\begin{equation}
\begin{subequations}\label{Plus_E1}
    \begin{align}
    \frac{1}{\eta}\Delta \phi' &= \norm{x_t' - x_t^*}^2 - \norm{x_{t-1} - x_{t-1}^*}^2\nonumber\\
    &= (\norm{x_t' - x_t^*} - \norm{x_{t-1} - x_{t-1}^*})(\norm{x_t' - x_t^*} + \norm{x_{t-1} - x_{t-1}^*})\nonumber\\
    &\leq (\norm{x_t' - x_{t-1}} + \norm{x_t^* - x_{t-1}^*})\big(\norm{x_t' - x_{t-1}} + \norm{x_t^* - x_{t-1}^*} + 2\norm{x_t' - x_t^*}\big)\label{Plus_E1:S3}\\
    &= (\norm{x_t' - x_{t-1}} + \norm{x_t^* - x_{t-1}^*})^2 + 2(\norm{x_t' - x_{t-1}} + \norm{x_t^* - x_{t-1}^*})\norm{x_t' - x_t^*}\nonumber\\
    &\leq 2\norm{x_t' - x_{t-1}}^2 + 2\norm{x_t^* - x_{t-1}^*}^2 + 2(\norm{x_t' - x_{t-1}} + \norm{x_t^* - x_{t-1}^*})\norm{x_t' - x_t^*}\label{Plus_E1:S5}\\
    &\leq 4M_t' + 4M_t^* + 2(\sqrt{2M_t'} + \sqrt{2M_t^*})\left(\sqrt{\frac{2H_t'}{m}} + \sqrt{\frac{2H_t^*}{m}}\right)\label{Plus_E1:S6}\\
    &\leq 4M_t' + 4M_t^* + \sqrt{\frac{1}{m}}\big((\sqrt{2M_t'} + \sqrt{2M_t^*})^2 + (\sqrt{2H_t'} + \sqrt{2H_t^*})^2\big)\label{Plus_E1:S7}\\
    &\leq 4M_t' + 4M_t^* + \sqrt{\frac{1}{m}}\big((4M_t' + 4M_t^*) + (4H_t' + 4H_t^*)\big)\label{Plus_E1:S8}\\
    &= \left(4 + 4\sqrt{\frac{1}{m}}\right)M_t' + \left(4 + 4\sqrt{\frac{1}{m}}\right)M_t^* + 4\sqrt{\frac{1}{m}}H_t' + 4\sqrt{\frac{1}{m}}H_t^*,\nonumber
    \end{align}
\end{subequations}
%\end{equation}
where we use the triangle inequality in line \eqref{Plus_E1:S3}; the generalized mean inequality in lines \eqref{Plus_E1:S5}, \eqref{Plus_E1:S7} and \eqref{Plus_E1:S8} and inequality \eqref{OBDPlusE1} in line \eqref{Plus_E1:S6}.

Using the OBD's balance condition $M_t' = \gamma H_t'$ and the assumption $H_t' \leq H_t^*$ based on inequality \eqref{Plus_E1}, we have
\begin{equation*}
    \begin{aligned}
    \frac{1}{\eta}\Delta \phi' &\leq \left(4 + 4\sqrt{\frac{1}{m}}\right)\gamma H_t' + \left(4 + 4\sqrt{\frac{1}{m}}\right)M_t^* + 4\sqrt{\frac{1}{m}}H_t' + 4\sqrt{\frac{1}{m}}H_t^*\\
    &\leq \left(4 + 4\sqrt{\frac{1}{m}}\right)\gamma H_t^* + \left(4 + 4\sqrt{\frac{1}{m}}\right)M_t^* + 8\sqrt{\frac{1}{m}}H_t^*.
    \end{aligned}
\end{equation*}
Notice that by the triangle inequality and the generalized mean inequality, we have that 
\[M_t = \frac{1}{2}\norm{x_t - x_{t-1}}^2 \leq \frac{1}{2}(\norm{x_t' - x_{t-1}} + \norm{x_t - x_t'})^2 \leq \frac{1}{2}(2\norm{x_t' - x_{t-1}}^2 + 2\norm{x_t - x_t'}^2).\]
Remember that since $\mu = 1$, we have $\norm{x_t - x_t'}^2 = m\norm{x_t' - v_t}^2$. Using this fact, we derive the following bound on $H_t + M_t + \Delta\phi$:
\begin{subequations}\label{Plus_E2}
    \begin{align}
    H_t + M_t + \Delta \phi \leq{}& H_t' + \frac{1}{2}\left(2\norm{x_t' - x_{t-1}}^2 + 2\norm{x_t - x_t'}^2\right) \nonumber\\
    &+ \eta (\norm{x_t - x_t^*}^2 - \norm{x_t' - x_t^*}^2) + \Delta \phi'\nonumber\\
    \leq{}& H_t' + \big(2M_t' + m \norm{x_t' - v_t}^2\big) \nonumber\\
    &+ \left( \eta\left(1 + \frac{1}{\sqrt{m}}\right)\norm{x_t - x_t'}^2 + \eta \sqrt{m}\norm{x_t' - x_t^*}^2\right) + \Delta \phi'\label{Plus_E2:S2}\\
    \leq{}& H_t' + \big(2M_t' + m \norm{x_t' - v_t}^2\big) \nonumber\\
    &+ \left( \eta\left(1 + \frac{1}{\sqrt{m}}\right)m\norm{x_t' - v_t}^2 + \eta \sqrt{m}\left(2\norm{x_t' - v_t}^2 + 2\norm{x_t^* - v_t}^2\right) \right) \nonumber\\
    &+ \Delta \phi'\label{Plus_E2:S3}\\
    \leq{}& H_t' + (2M_t' + 2H_t') + \left(\eta\left(1 + \frac{1}{\sqrt{m}}\right)2H_t' + \eta \sqrt{m}\left(\frac{4H_t'}{m} + \frac{4H_t^*}{m}\right)\right) + \Delta \phi'\label{Plus_E2:S4}\\
    ={}& (3 + 2\eta + \frac{6\eta}{\sqrt{m}})H_t' + 2M_t' + 4\eta \frac{H_t^*}{\sqrt{m}} + \Delta \phi'\nonumber\\
    ={}& \Big(3 + 2\eta + \frac{6\eta}{\sqrt{m}} + 2\gamma \Big)H_t' + 4\eta \frac{H_t^*}{\sqrt{m}} + \Delta \phi'\nonumber\\
    \leq{}& \Big(3 + 2\eta + \frac{6\eta}{\sqrt{m}} + 2\gamma \Big)H_t^* + 4\eta \frac{H_t^*}{\sqrt{m}} + \Delta \phi'\label{Plus_E2:S7}\\
    ={}& \Big(3 + 2\eta + \frac{10\eta}{\sqrt{m}} + 2\gamma \Big)H_t^* + \Delta \phi', \nonumber
    \end{align}
\end{subequations}
where we use Lemma \ref{thm:quad} in line \eqref{Plus_E2:S2}; the triangle inequality in line \eqref{Plus_E2:S3}; $m$-strongly convexity of $f_t$ in line \eqref{Plus_E2:S4}; and the assumption $H_t' \leq H_t^*$ in line \eqref{Plus_E2:S7}. 

Combining inequalities \eqref{Plus_E1} and \eqref{Plus_E2}, we obtain
\begin{equation}
    \begin{aligned}
    H_t + M_t + \Delta \phi \leq \big( 3 + 2\eta + 2\gamma + 4\eta \gamma + \frac{\eta}{\sqrt{m}}(18 + 4\gamma)\big)H_t^* + \eta(4 + 4\sqrt{\frac{1}{m}})M_t^*.
    \end{aligned}
\end{equation}

\subsection*{Case 2: $H_t' \geq H_t^*$}
In this case, we prove that for any $x_t^*, x_{t-1}^* \in \mathbb{R}^d$, we have
\begin{equation}\label{OBDPlusE2}
    H_t + M_t + \Delta \phi \leq \frac{C}{\sqrt{m}}(H_t^* + M_t^*),
\end{equation}
for some positive constant $C$.

In the proof, we use $D_1, D_2, \cdots, D_d$ to represent the $d$ axes in the coordinate system.

\input{Figure_G-OBD-Proof1.tex}

As shown in Figure \ref{figure:OBDPlusCase2}, without loss of generality, let $v_t = (0, 0, \cdots, 0), x_t' = (h_1, h_2, 0, \cdots, 0)$ and $D_2 = h_2$ be the projection hyper plane, where $h_1 \geq 0, h_2 \geq 0$. 
And let $l = \norm{x_{t-1} - x_t'} > 0$. 
Note that our analysis still holds in one-dimension because we can restrict ourselves to the $D_2$ axis.

Then we know $x_{t-1} = (h_1, h_2 + l, 0, \cdots, 0), x_t = (h_1(1 - \sqrt{m}), h_2(1 - \sqrt{m}), 0, \cdots, 0)$. Since we know $x_t^*$ must lie below the projection hyper plane, we can let $x_t^* = (x, h_2 - y, a_3, a_4, \cdots, a_d)$, where $y > 0$.

Now we show that it suffices to prove the statement when $x_{t-1}^*$ is on the line segment $x_t^* x_{t-1}$. 
Suppose $x_{t-1}^*$ is not on the line segment $x_t^* x_{t-1}$.
If $\norm{x_{t-1}^* - x_{t-1}} > \norm{x_t^* - x_{t-1}}$, by moving $x_{t-1}^*$ to $x_t^*$, $\Delta \phi$ increases and $M_t^*$ decreases.
Otherwise, we can choose a point $K$ on line segment $x_t^* x_{t-1}$ such that $\norm{K - x_{t-1}} = \norm{x_{t-1}^* - x_{t-1}}$. By moving $x_{t-1}^*$ to $K$, $\Delta \phi$ remains unchanged and $M_t^*$ decreases.
Therefore if inequality \eqref{OBDPlusE2} holds for $x^*_{t-1}$ on the segment $x^*_t x_{t_1}$, then it must also hold for any other $x^*_{t-1} \in \mathbb{R}^d$.

Now we suppose $x_{t-1}^*$ is on the line segment $x_t^* x_{t-1}$, and $\norm{x_t^* - x_{t-1}^*} = \lambda \norm{x_t^* - x_{t-1}}$.

Recall that we set $\gamma = 1$, so $M_t' = \gamma H_t' = H_t'$. It follows that
\[ M_t \leq l^2 + \norm{x_t - x_t'}^2 = l^2 + m(h_1^2 + h_2^2) \leq l^2 + 2H_t' = l^2 + 2M_t' \leq 2l^2,\]
and
\[H_t \leq H_t' = M_t' = \frac{l^2}{2}.\]
We can separate $\Delta \phi$ into two parts:
\[\frac{\Delta \phi}{\eta} = \left(\norm{x_t^* - x_t}^2 - \norm{x_t^* - x_{t-1}}^2\right) + \left(\norm{x_t^* - x_{t-1}}^2 - \norm{x_{t-1}^* - x_{t-1}}^2\right).\]
For convenience, we define
\[\Delta \phi_1 := \left(\norm{x_t^* - x_t}^2 - \norm{x_t^* - x_{t-1}}^2\right),\]
and
\[\Delta \phi_2 := \left(\norm{x_t^* - x_{t-1}}^2 - \norm{x_{t-1}^* - x_{t-1}}^2\right).\]
We further notice that from the triangle inequality,
\begin{equation}\label{OBDPlusE3}
    \Delta \phi_2 \le (1 - (1 - \lambda)^2)\norm{x_t^* - x_{t-1}}^2 = \lambda (2 - \lambda)\left((x - h_1)^2 + (y + l)^2 + \sum_{i=3}^d a_i^2\right).
\end{equation}
Now we express $M_t^*$ and $H_t^*$ in terms of the variables we define, which are
\begin{equation}\label{OBDPlusE4}
    M_t^* = \frac{1}{2}(\lambda \norm{x_t^* - x_{t-1}})^2 = \frac{\lambda^2}{2} \left((x - h_1)^2 + (y + l)^2 + \sum_{i=3}^d a_i^2\right),
\end{equation}
and
\begin{equation}\label{OBDPlusE5}
    H_t^* \geq \frac{m}{2}\norm{x_t^* - v_t}^2 = \frac{m}{2}\left(x^2 + (h_2 - y)^2 + \sum_{i=3}^d a_i^2\right).
\end{equation}
We also expand $\Delta \phi_1$:
\begin{equation}\label{d2}
\begin{aligned}
    \Delta \phi_1 &= \norm{x_t^* - x_t}^2 - \norm{x_t^* - x_{t-1}}^2\\
    &= (x - h_1 + h_1\sqrt{m})^2 + (y - h_2\sqrt{m})^2 + \sum_{i=3}^d a_i^2 - (x - h_1)^2 - (y + l)^2 - \sum_{i=3}^d a_i^2\\
    &= \left((x - h_1 + h_1\sqrt{m})^2 - (x - h_1)^2\right) + \left((y - h_2\sqrt{m})^2 - (y + l)^2\right)\\
    &= h_1\sqrt{m}(2x - 2h_1 + h_1\sqrt{m}) - (h_2\sqrt{m} + l)(2y + l - h_2\sqrt{m})\\
    &= 2x h_1\sqrt{m} - 2h_1^2\sqrt{m} + h_1^2 m - 2y(h_2\sqrt{m} + l) - l^2 + h_2^2 m.
\end{aligned}
\end{equation}
Using the condition that $m \leq \frac{9}{64} < 1$, we derive the following bound:
\begin{equation}\label{d3}
\begin{aligned}
    \Delta \phi_1 &\le 2x h_1\sqrt{m} - 2h_1^2\sqrt{m} + h_1^2 m - 2y(h_2\sqrt{m} + l) - l^2 + h_2^2 \sqrt{m}\\
    &= 2x h_1\sqrt{m} - 2h_1^2\sqrt{m} + h_1^2 m + \sqrt{m}(h_2 - y)^2 - \sqrt{m}y^2 - 2yl - l^2.
\end{aligned}
\end{equation}
Substituting equations \eqref{OBDPlusE4} and \eqref{OBDPlusE5} into inequality \eqref{OBDPlusE2}, we know that it suffices to show that for some constant $C$,
\begin{equation}\label{OBDPlusT1Target1}
    M_t + H_t + \eta \Delta \phi_1 + \eta \Delta \phi_2 \leq \frac{C}{\sqrt{m}}\Bigg( \frac{m}{2}\Big(x^2 + (h_2 - y)^2 + \sum_{i=3}^d a_i^2\Big) + \frac{\lambda^2}{2} \Big((x - h_1)^2 + (y + l)^2 + \sum_{i=3}^d a_i^2\Big)\Bigg).
\end{equation}

\subsubsection*{Subcase 2.1: $\lambda \leq \frac{\sqrt{m}}{2}$}
We can bound equation \eqref{OBDPlusE3} as follows:
\begin{equation}\label{d1}
\begin{aligned}
    \Delta \phi_2 &= \lambda (2 - \lambda)\left((x - h_1)^2 + (y + l)^2 + \sum_{i=3}^d a_i^2\right)\\
    &\leq \sqrt{m}(x - h_1)^2 + \sqrt{m}(y + l)^2 + \sqrt{m}\sum_{i=3}^d a_i^2\\
    &= \sqrt{m}x^2 - 2\sqrt{m}x h_1 + \sqrt{m}h_1^2 + \sqrt{m}y^2 + 2\sqrt{m}yl + \sqrt{m}l^2 + \sqrt{m}\sum_{i=3}^d a_i^2.
\end{aligned}
\end{equation}
Summing inequalities \eqref{d3} and \eqref{d1}, we get
\begin{subequations}\label{Plus_Sub1_E1}
\begin{align}
\Delta \phi_1 + \Delta \phi_2 &\leq \sqrt{m}x^2 + (-h_1^2\sqrt{m} + h_1^2 m) + \sqrt{m}(h_2 - y)^2\nonumber\\
&\quad+ (2\sqrt{m}yl - 2yl) + (\sqrt{m}l^2 - l^2) + \sqrt{m}\sum_{i=3}^d a_i^2\nonumber\\
&\leq \sqrt{m}x^2 + 0 + \sqrt{m}(h_2 - y)^2 + 0 - \frac{5}{8}l^2 + \sqrt{m}\sum_{i=3}^d a_i^2\label{Plus_Sub1_E1:S2}\\
&\leq \sqrt{m}x^2 + \sqrt{m}(h_2 - y)^2 - \frac{5}{8}l^2 + \sqrt{m}\sum_{i=3}^d a_i^2,\nonumber
\end{align}
\end{subequations}
where we use the condition that $m \leq \frac{9}{64}$ in line \eqref{Plus_Sub1_E1:S2}. We further obtain
\begin{equation*}
\begin{aligned}
M_t + H_t + \eta(\Delta \phi_1 + \Delta \phi_2) &\leq 2l^2 + \frac{l^2}{2} + \eta \left( \sqrt{m}x^2 + \sqrt{m}(h_2 - y)^2 - \frac{5}{8}l^2 + \sqrt{m}\sum_{i=3}^d a_i^2 \right)\\
&= \frac{5l^2}{2} + 4 \left( \sqrt{m}x^2 + \sqrt{m}(h_2 - y)^2 - \frac{5}{8}l^2 + \sqrt{m}\sum_{i=3}^d a_i^2 \right)\\
&= 4 \left( \sqrt{m}x^2 + \sqrt{m}(h_2 - y)^2 + \sqrt{m}\sum_{i=3}^d a_i^2 \right).
\end{aligned}
\end{equation*}
Therefore, for $C \geq 8$, we have
\[ M_t + H_t + \eta \Delta \phi_1 + \eta \Delta \phi_2 \leq \frac{C}{\sqrt{m}}\Bigg( \frac{m}{2}\Big(x^2 + (h_2 - y)^2 + \sum_{i=3}^d a_i^2\Big) + \frac{\lambda^2}{2} \Big((x - h_1)^2 + (y + l)^2 + \sum_{i=3}^d a_i^2\Big) \Bigg),\]
which establishes inequality \eqref{OBDPlusT1Target1}.

\subsubsection*{Subcase 2.2: $\lambda \geq \frac{\sqrt{m}}{2}$}
Notice that when $C \geq 32$, we have
\[\frac{C}{2\sqrt{m}}\lambda^2 \geq \frac{16}{\sqrt{m}}\lambda^2 \geq \frac{16}{\sqrt{m}}\cdot \frac{\sqrt{m}}{2}\lambda = 8\lambda \geq 4\lambda (2 - \lambda) = \eta \lambda (2 - \lambda).\]
Substituting this inequality into equation \eqref{OBDPlusE3}, we know that for $C \geq 32$,
\begin{equation}\label{d4}
\begin{aligned}
\eta \Delta \phi_2 \leq \frac{C}{\sqrt{m}}\cdot \frac{\lambda^2}{2} \left((x - h_1)^2 + (y + l)^2 + \sum_{i=3}^d a_i^2\right).
\end{aligned}
\end{equation}
We can further bound inequality \eqref{d3}:
\begin{equation*}
\begin{aligned}
\Delta \phi_1 &\leq 2x h_1\sqrt{m} - 2h_1^2\sqrt{m} + h_1^2 m + \sqrt{m}(h_2 - y)^2 - \sqrt{m}y^2 - 2yl - l^2\\
&\leq \sqrt{m}x^2 + \sqrt{m}h_1^2 - 2h_1^2\sqrt{m} + h_1^2 m + \sqrt{m}(h_2 - y)^2 - l^2\\
&\leq \sqrt{m}x^2 + \sqrt{m}(h_2 - y)^2 - l^2,
\end{aligned}
\end{equation*}
where we apply the AM-GM inequality in step 2 and use the condition $m < 1$ in step 3.

Therefore we have
\begin{equation}\label{d5}
\begin{aligned}
H_t + M_t + \eta \Delta \phi_1 &\leq \frac{5l^2}{2} + 4(\sqrt{m}x^2 + \sqrt{m}(h_2 - y)^2 - l^2)\\
&\leq 4(\sqrt{m}x^2 + \sqrt{m}(h_2 - y)^2).
\end{aligned}
\end{equation}
Summing inequalities \eqref{d5} and \eqref{d4}, we yield that for $C\geq 32$,
\[ M_t + H_t + \eta \Delta \phi_1 + \eta \Delta \phi_2 \leq \frac{C}{\sqrt{m}}\Bigg( \frac{m}{2}\Big(x^2 + (h_2 - y)^2 + \sum_{i=3}^d a_i^2\Big) + \frac{\lambda^2}{2} \Big((x - h_1)^2 + (y + l)^2 + \sum_{i=3}^d a_i^2\Big) \Bigg),\]
which establishes inequality \eqref{OBDPlusT1Target1}.

Combining all cases above, we conclude that G-OBD is an $O(\frac{1}{\sqrt{m}})$-competitive algorithm.

%\end{proof}

%% file: Figure_G-OBD-Proof1.tex
\begin{figure}
    \begin{center}
        \includegraphics{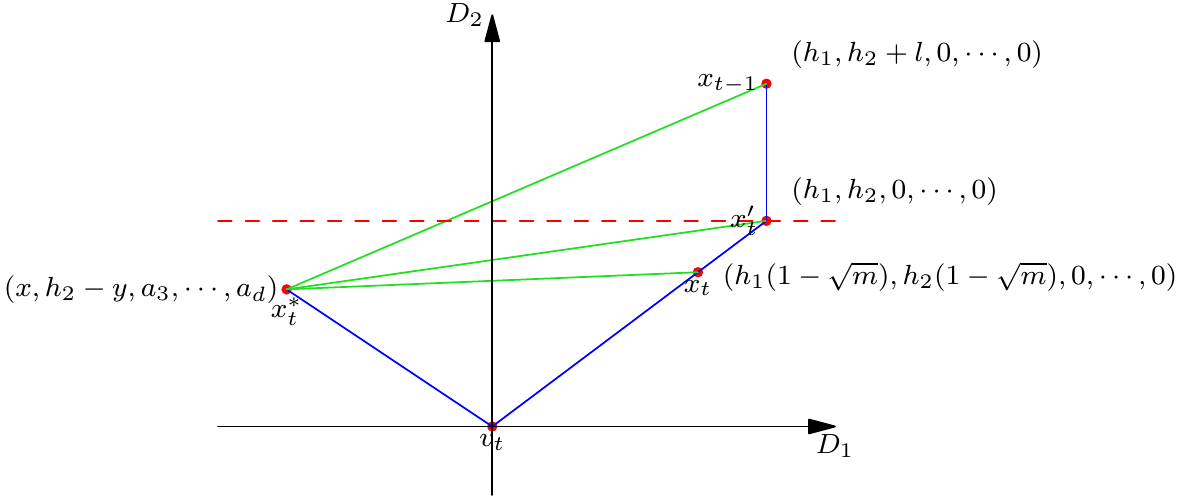}
    \end{center}
\caption{\emph{Starting at $x_{t-1}$, G-OBD first does projection on to the $H_t'$ level set (red dashed line) in the first phase. The projection point is $x_t'$. Then G-OBD moves toward the minimizer to obtain point $x_t$ in the second phase. Let the minimizer $v_t$ be the origin. Notice that the three points $x_{t-1}, x_t', v_t$ defines a plane $S$. Without loss of generality, we can let axis $D_2$ be parallel to line $x_t'x_{t-1}$; and let axis $D_1$ be parallel to the projection hyperplane.}}
\label{figure:OBDPlusCase2}
\end{figure}

%% file: MainT1.tex
To prove Theorem \ref{MainT1} we make use of Lemma \ref{DualNormT1} and \ref{GWGL1}.

%\begin{proof}[Proof of Theorem \ref{MainT1}]
Our approach is to make use of strong convexity and properties of Bregman Divergences to derive an inequality in the form of $H_t + M_t + \Delta \phi \leq C(H_t^* + M_t^*)$ for some positive constant $C$, where $\Delta \phi$ is the change in potential, which we will define later. The constant $C$ is then an upper bound for the competitive ratio.

To begin, recall that $h$ is assumed to be $\alpha-$strongly convex and $\beta-$strongly smooth with respect to norm $\norm{\cdot}$. Thus we can give a trivial bound on Bregman Divergence, namely
\begin{equation}\label{MainT1E0}
    \forall x, y, \frac{\alpha}{2}\norm{x - y}^2 \leq D_h(x || y) \leq \frac{\beta}{2}\norm{x - y}^2.
\end{equation}

Recall that the update rule in Algorithm \ref{alg:Breg} can be stated as:
\[ x_t = \argmin_x f_t(x) + \lambda_1 D_h(x || x_{t-1}) + \lambda_2 D_h(x || v_t). \]

Since the function $f_t(x) + \lambda_1 D_h(x || x_{t-1}) + \lambda_2 D_h(x || v_t)$ is strongly convex, the minimizer $x_t$ exists and is unique. Furthermore, it must satisfy the first-order condition
\[ \nabla f_t(x_t) + \lambda_1 (\nabla h(x_t) - \nabla h(x_{t-1})) + \lambda_2 (\nabla h(x_t) - \nabla h(v_t)) = 0. \]

Further, since $f_t(x)$ is $m$-strongly convex, we have
\begin{equation}\label{MainT1_1}
    \begin{aligned}
    f_t(x_t^*) &\geq f_t(x_t) + \langle \nabla f_t(x_t), x_t^* - x_t\rangle + \frac{m}{2}\norm{x_t^* - x_t}^2\\
    &= f_t(x_t) - \lambda_1\langle \nabla h(x_{t-1}) - \nabla h(x_t), x_t - x_t^* \rangle\\
    &\phantom{==}- \lambda_2 \langle \nabla h(v_t) - \nabla h(x_t), x_t - x_t^* \rangle + \frac{m}{2}\norm{x_t^* - x_t}^2.
    \end{aligned}
\end{equation}
Using Lemma \ref{GWGL1}, we obtain
\[ \langle \nabla h(x_{t-1}) - \nabla h(x_t), x_t - x_t^* \rangle = D_h(x_t^* || x_{t-1}) - D_h(x_t^* || x_t) - D_h(x_t || x_{t-1}), \]
and
\[ \langle \nabla h(v_t) - \nabla h(x_t), x_t - x_t^* \rangle = D_h(x_t^* || v_t) - D_h(x_t^* || x_t) - D_h(x_t || v_t).\]
Substituting the two above identities into inequality \eqref{MainT1_1}, we get
\begin{equation*}
\begin{aligned}
    &f_t(x_t) + \lambda_1 D_h(x_t || x_{t-1}) + \lambda_2 D_h(x_t || v_t) + (\lambda_1 + \lambda_2)D_h(x_t^* || x_t) + \frac{m}{2}\norm{x_t^* - x_t}^2\\
    \leq{}& f_t(x_t^*) + \lambda_1 D_h(x_t^* || x_{t-1}) + \lambda_2 D_h(x_t^* || v_t).
\end{aligned}
\end{equation*}
It follows that
\begin{equation}\label{MainT1E1}
    \begin{aligned}
    &f_t(x_t) + \lambda_1 D_h(x_t || x_{t-1}) + (\lambda_1 + \lambda_2)D_h(x_t^* || x_t) + \frac{m}{2}\norm{x_t^* - x_t}^2\\
    \leq{}& f_t(x_t^*) + \lambda_1 D_h(x_t^* || x_{t-1}) + \lambda_2 D_h(x_t^* || v_t).
    \end{aligned}
\end{equation}
We define the potential function as $\phi(x_t, x_t^*) = (\lambda_1 + \lambda_2) D_h(x_t^* || x_t) + \frac{m}{2}\norm{x_t^* - x_t}^2$, and let $\Delta \phi = \phi(x_t, x_t^*) - \phi(x_{t-1}, x_{t-1}^*)$. Applying this notation to inequality \eqref{MainT1E1} and rearranging terms, we obtain
\begin{equation}\label{MainT1E4}
    \begin{aligned}
    &H_t + \lambda_1 M_t + \Delta \phi\\
    \leq{}& \left(H_t^* + \lambda_2 D_h(x_t^* || v_t)\right) + \lambda_1 D_h(x_t^* || x_{t-1}) - (\lambda_1 + \lambda_2) D_h(x_{t-1}^* || x_{t-1}) - \frac{m}{2}\norm{x_{t-1}^* - x_{t-1}}^2.
    \end{aligned}
\end{equation}
Using Lemma \ref{DualNormT1}, we get
\begin{equation}\label{MainT1E2}
    \frac{1}{2\beta}\norm{\nabla h(x_{t-1}) - \nabla h(x_{t-1}^*)}_*^2 \leq D_h(x_{t-1}^* || x_{t-1}),
\end{equation}
and
\begin{equation}\label{MainT1E3}
    \norm{\nabla h(x_{t-1}) - \nabla h(x_{t-1}^*)}_* \leq \beta \norm{x_{t-1} - x_{t-1}^*}.
\end{equation}
Using Lemma \ref{GWGL1} and the two above inequalities, we get
\begin{subequations}\label{MainT1E5}
    \begin{align}
    &\lambda_1 D_h(x_t^* || x_{t-1}) - (\lambda_1 + \lambda_2) D_h(x_{t-1}^* || x_{t-1}) - \frac{m}{2}\norm{x_{t-1}^* - x_{t-1}}^2\nonumber\\
    ={}&\lambda_1 \big(D_h(x_t^* || x_{t-1}) - D_h(x_{t-1}^* || x_{t-1})\big) - \lambda_2 D_h(x_{t-1}^* || x_{t-1}) - \frac{m}{2}\norm{x_{t-1}^* - x_{t-1}}^2\label{MainT1E5:S2}\\
    ={}& \lambda_1 D_h(x_t^* || x_{t-1}^*) + \lambda_1 \langle \nabla h(x_{t-1}) - \nabla h(x_{t-1}^*), x_{t-1}^* - x_t^* \rangle\nonumber\\ 
    &- \lambda_2 D_h(x_{t-1}^* || x_{t-1}) - \frac{m}{2}\norm{x_{t-1}^* - x_{t-1}}^2\label{MainT1E5:S3}\\
    \leq{}& \lambda_1 D_h(x_t^* || x_{t-1}^*) + \lambda_1 \norm{\nabla h(x_{t-1}) - \nabla h(x_{t-1}^*)}_* \norm{x_{t-1}^* - x_t^*}\nonumber\\ 
    &- \lambda_2 D_h(x_{t-1}^* || x_{t-1}) - \frac{m}{2}\norm{x_{t-1}^* - x_{t-1}}^2\label{MainT1E5:S4}\\
    \leq{}& \lambda_1 D_h(x_t^* || x_{t-1}^*) + \frac{\lambda_2 \beta + m}{2\beta^2}\norm{\nabla h(x_{t-1}) - \nabla h(x_{t-1}^*)}_*^2 +  \frac{\lambda_1^2 \beta^2}{2(\lambda_2 \beta + m)}\norm{x_{t-1}^* - x_t^*}^2\nonumber\\
    &- \lambda_2 D_h(x_{t-1}^* || x_{t-1}) - \frac{m}{2}\norm{x_{t-1}^* - x_{t-1}}^2\nonumber\\
    ={}& \lambda_1 D_h(x_t^* || x_{t-1}^*) + \frac{\lambda_1^2\beta^2}{2(\lambda_2 \beta + m)}\norm{x_{t-1}^* - x_t^*}^2\nonumber\\ 
    &+ \left(\frac{\lambda_2}{2\beta}\norm{\nabla h(x_{t-1}) - \nabla h(x_{t-1}^*)}_*^2 - \lambda_2 D_h(x_{t-1}^* || x_{t-1}) \right)\nonumber\\
    &+ \left( \frac{m}{2\beta^2}\norm{\nabla h(x_{t-1}) - \nabla h(x_{t-1}^*)}_*^2 - \frac{m}{2}\norm{x_{t-1}^* - x_{t-1}}^2 \right)\label{MainT1E5:S6}\\
    \leq{}& \lambda_1 D_h(x_t^* || x_{t-1}^*) + \frac{\lambda_1^2\beta^2}{2(\lambda_2 \beta + m)}\norm{x_{t-1}^* - x_t^*}^2\nonumber\\
    \leq{}& \lambda_1\left(1 + \frac{\lambda_1\beta^2}{\alpha(\lambda_2 \beta + m)}\right)D_h(x_t^* || x_{t-1}^*),\label{MainT1E5:S7}
    \end{align}
\end{subequations}
where we use Lemma \ref{GWGL1} in line \eqref{MainT1E5:S2}; Cauchy-Schwartz inequality in line \eqref{MainT1E5:S3}; the AM-GM inequality in the line \eqref{MainT1E5:S4}; inequalities \eqref{MainT1E2} and \eqref{MainT1E3} in line \eqref{MainT1E5:S6}; 
and inequality \eqref{MainT1E0} in line \eqref{MainT1E5:S7}. 

Substituting inequality \eqref{MainT1E5} into inequality \eqref{MainT1E4}, we obtain
\[ H_t + \lambda_1 M_t + \Delta \phi \leq \big(H_t^* + \lambda_2 D_h(x_t^* || v_t)\big) + \lambda_1\left(1 + \frac{\lambda_1\beta^2}{\alpha(\lambda_2 \beta + m)}\right)M_t^*. \]
Using inequality \eqref{MainT1E0} and the fact that $f_t$ is $m$-strongly convex, we obtain
\begin{equation*}
    \begin{aligned}
    \lambda_2 D_h(x_t^* || v_t) &\leq \frac{\lambda_2\beta}{2}\norm{x_t^* - v_t}^2\leq \frac{\lambda_2\beta}{m}H_t^*.
    \end{aligned}
\end{equation*}

Therefore we have
\[ H_t + \lambda_1 M_t + \Delta \phi \leq (1 + \frac{\lambda_2\beta}{m})H_t^* + \lambda_1\left(1 + \frac{\lambda_1\beta^2}{\alpha(\lambda_2 \beta + m)}\right)M_t^*.\]
Since $0 < \lambda_1 \leq 1$, we have
\[H_t + M_t + \frac{1}{\lambda_1}\Delta \phi \leq \frac{H_t + \lambda_1 M_t + \Delta \phi}{\lambda_1} \leq \frac{m + \lambda_2 \beta}{m\lambda_1} H_t^* + \left(1 + \frac{\beta^2}{\alpha}\cdot \frac{\lambda_1}{\lambda_2 \beta + m}\right)M_t^*.\]

Theorem \ref{MainT1} follows from summing the above inequality over all timesteps $t$.

%\end{proof}

%% file: thm-opt-cr.tex
% \haoyuan{TODO: explain that this is a special case of Bregman, but a direct proof allows us to relax the differentiability condition.}

When $h(x) = \frac{1}{2} \norm{x}_2^2$, the Bregman Divergence $D_h(x || y)$ is equal to the squared $\ell_2$ norm $\frac{1}{2}\norm{x - y}_2^2$.
Hence, setting $h(x) = \frac{1}{2} \norm{x}_2^2$ in Algorithm \ref{alg:Breg} gives us R-OBD in the squared $\ell_2$ setting.
In this section, we present a separate proof of Regularized OBD with squared $\ell_2$ norm, in order to remove the assumption that the hitting costs $\{f_t\}$ are differentiable.

\begin{theorem} \label{thm-optimal-cr}
Consider hitting cost functions that are $m$-strongly convex with respect to $\ell_2$ norm and movement costs given by $\frac{1}{2}\norm{x_t - x_{t-1}}_2^2$. There exists a choice $\lambda_1, \lambda_2$ such that the competitive ratio of Regularized OBD matches the lower bound proved in Theorem \ref{GeneralLowerT1}, i.e. the competitive ratio is at most $\frac{1}{2}\left(1 + \sqrt{1 + \frac{4}{m}}\right)$.
\end{theorem}

This result follows from the more general bound in Theorem \ref{MainT0} below, which describes the competitive ratio of Algorithm \ref{alg:Breg} as a function of $\lambda_1, \lambda_2$.

\begin{theorem}\label{MainT0}
Consider hitting cost functions that are $m$-strongly convex with respect to $\ell_2$ norm and movement costs given by $\frac{1}{2}\norm{x_t - x_{t-1}}_2^2$. Regularized-OBD (Algorithm \ref{alg:Breg} with $h(x) = \frac{1}{2}\norm{x}_2^2$) with parameters $1 \geq \lambda_1 > 0, \lambda_2 \geq 0$ has competitive ratio at most
\[ \max \left(\frac{m + \lambda_2}{\lambda_1}\cdot \frac{1}{m}, 1 + \frac{\lambda_1}{\lambda_2 + m}\right).\]
\end{theorem}

Notice that Theorem \ref{thm-optimal-cr} follows immediately by setting $\frac{m + \lambda_2}{\lambda_1} =  \frac{m}{2}\left(1 + \sqrt{1 + \frac{4}{m}}\right)$ in Theorem \ref{MainT0}.

Before proving Theorem \ref{MainT0}, we first prove a teechnical lemma which gives a lower bound of the value of hitting cost as a function of the distance to the minimizer.
\begin{lemma}\label{MainL0}
If $f: \mathcal{X} \to \mathbb{R}$ is a $m$-strongly convex function with respect to some norm $\norm{\cdot}$, and $v$ is the minimizer of f (i.e. $v = \argmin_{x \in \mathcal{X}}f(x)$), then we have $\forall x \in \mathcal{X}$,
\[ f(x) \geq f(v) + \frac{m}{2}\norm{x - v}^2. \]
\end{lemma}
\begin{proof}%[Proof of Lemma \ref{MainL0}]
By the definition of $m$-strongly convex, we obtain that $\forall \alpha \in (0, 1)$,
\begin{equation}\label{MainL0E1}
    f(\alpha x + (1 - \alpha)v) \leq \alpha f(x) + (1 - \alpha)f(v) - \frac{m}{2}\alpha (1 - \alpha) \norm{x - v}^2.
\end{equation}
Notice that $f(v) \leq f(\alpha x + (1 - \alpha)v)$. Combining this with inequality \eqref{MainL0E1}, we obtain that $\forall \alpha \in (0, 1)$,
\[ f(v) \leq \alpha f(x) + (1 - \alpha)f(v) - \frac{m}{2}\alpha (1 - \alpha) \norm{x - v}^2. \]
Rearranging the terms, we observe that $\forall \alpha \in (0, 1)$,
\[ f(x) \geq f(v) + \frac{m}{2}(1 - \alpha)\norm{x - v}^2.\]
Therefore
\[ f(x) \geq \lim_{\alpha \to 0^+}\left( f(v) + \frac{m}{2}(1 - \alpha)\norm{x - v}^2\right) = f(v) + \frac{m}{2}\norm{x - v}^2.\]
\end{proof}

\input{MainT0.tex}

\begin{comment}
Notice that R-OBD (Algorithm \ref{alg:Breg}) with $h(x) = \frac{1}{2}\norm{x}_2^2$ is equivalent to Algorithm \ref{alg:L2Quadratic}.

\begin{algorithm}
\caption{Regularized OBD: $\ell_2$ quadratic setting}\label{alg:L2Quadratic}
\begin{algorithmic}[1]
\Procedure{R-OBD}{$f_t, x_{t-1}$}\Comment{Procedure to select $x_t$}
\State $v_t \gets \argmin_x f_t(x)$
\State $x_t \gets \argmin_x f_t(x) + \frac{\lambda_1}{2}\norm{x - x_{t-1}}_2^2 + \frac{\lambda_2}{2}\norm{x - v_t}_2^2$
\State \textbf{return} $x_t$
\EndProcedure
\end{algorithmic}
\end{algorithm}
\end{comment}

%% file: MainT0.tex
Now we return to the proof of Theorem \ref{MainT0}.
\begin{proof}[Proof of Theorem \ref{MainT0}]
In the proof, we use the property of strongly convex to derive an inequality in the form of $H_t + M_t + \Delta \phi \leq C(H_t^* + M_t^*)$, where $\Delta \phi$ is the change in potential and $C$ is an upper bound for the competitive ratio.

Throughout the proof, we use $\norm{\cdot}$ to denote $\ell_2$ norm.

Notice that when $h(x) = \frac{1}{2}\norm{x}^2$, the update rule in Algorithm \ref{alg:Breg} is:
\[ x_t = \argmin_x f_t(x) + \frac{\lambda_1}{2}\norm{x - x_{t-1}}^2 + \frac{\lambda_2}{2}\norm{x - v_t}^2.\]

For convenience, we define
\[F_t(x) = f_t(x) + \frac{\lambda_1}{2}\norm{x - x_{t-1}}^2 + \frac{\lambda_2}{2}\norm{x - v_t}^2.\]
Since $f_t(x)$ is $m$-strongly convex,$\frac{\lambda_1}{2}\norm{x - x_{t-1}}^2$ is $\lambda_1$-strongly convex, and $\frac{\lambda_2}{2}\norm{x - v_t}^2$ is $\lambda_2$-strongly convex, $F_t(x)$ is $(m + \lambda_1 + \lambda_2)-$strongly convex. Since $x_t = \argmin_x F_t(x)$, by Lemma \ref{MainL0}, we obtain
\[F_t(x_t^*) \geq F_t(x_t) + \frac{m + \lambda_1 + \lambda_2}{2}\norm{x_t^* - x_t}^2,\]
which implies
\begin{equation}\label{MainT0E0}
\begin{aligned}
& H_t + \lambda_1 M_t + \frac{m + \lambda_1 + \lambda_2}{2}\norm{x_t^* - x_t}^2 \\
\le{}& H_t + \lambda_1 M_t + \frac{\lambda_2}{2}\norm{x_t - v_t}^2 + \frac{m + \lambda_1 + \lambda_2}{2}\norm{x_t^* - x_t}^2 \\
\le{}& H_t^* + \frac{\lambda_1}{2}\norm{x_t^* - x_{t-1}}^2 + \frac{\lambda_2}{2}\norm{x_t^* - v_t}^2.
\end{aligned}
\end{equation}

We define the potential function as $\phi(x_t, x_t^*) = \frac{m + \lambda_1 + \lambda_2}{2}\norm{x_t^* - x_t}^2$ and $\Delta \phi = \phi(x_t, x_t^*) - \phi(x_{t-1}, x_{t-1}^*)$. We then can rewrite inequality \eqref{MainT0E0} as
\begin{equation}\label{T0E1}
    H_t + \lambda_1 M_t + \Delta \phi \leq \left(H_t^* + \frac{\lambda_2}{2}\norm{x_t^* - v_t}^2\right) + \frac{\lambda_1}{2}\norm{x_t^* - x_{t-1}}^2 - \frac{m + \lambda_1 + \lambda_2}{2}\norm{x_{t-1}^* - x_{t-1}}^2.
\end{equation}
Additionally
\begin{subequations}\label{MainT0E1}
    \begin{align}
    &\frac{\lambda_1}{2}\norm{x_t^* - x_{t-1}}^2 - \frac{m + \lambda_1 + \lambda_2}{2}\norm{x_{t-1}^* - x_{t-1}}^2\nonumber\\
    \leq{}& \frac{\lambda_1}{2}\left(\norm{x_t^* - x_{t-1}^*} + \norm{x_{t-1}^* - x_{t-1}}\right)^2 - \frac{m + \lambda_1 + \lambda_2}{2}\norm{x_{t-1}^* - x_{t-1}}^2\label{MainT0E1:S2}\\
    ={}& \frac{\lambda_1}{2}\norm{x_t^* - x_{t-1}^*}^2 + \lambda_1\norm{x_t^* - x_{t-1}^*}\cdot \norm{x_{t-1}^* - x_{t-1}} - \frac{m + \lambda_2}{2}\norm{x_{t-1}^* - x_{t-1}}^2\nonumber\\
    \leq{}& \frac{\lambda_1}{2}\norm{x_t^* - x_{t-1}^*}^2 + \frac{\lambda_1^2}{2(m + \lambda_2)}\norm{x_t^* - x_{t-1}^*}^2 + \frac{m + \lambda_2}{2}\norm{x_{t-1}^* - x_{t-1}}^2\nonumber\\
    &- \frac{m + \lambda_2}{2}\norm{x_{t-1}^* - x_{t-1}}^2\label{MainT0E1:S4}\\
    ={}& \frac{\lambda_1(\lambda_1 + \lambda_2 + m)}{2(\lambda_2 + m)}\norm{x_{t-1}^* - x_{t-1}^*}^2\nonumber\\
    ={}& \lambda_1\left(1 + \frac{\lambda_1}{\lambda_2 + m}\right)M_t^*,\nonumber
    \end{align}
\end{subequations}
where we apply the triangle inequality in line \eqref{MainT0E1:S2} and AM-GM in line \eqref{MainT0E1:S4}.

Combining inequalities \eqref{T0E1} and \eqref{MainT0E1}, we obtain
\begin{equation}\label{T0E2}
    H_t + \lambda_1 M_t + \Delta \phi \leq \left(H_t^* + \frac{\lambda_2}{2}\norm{x_t^* - v_t}^2\right) + \lambda_1\left(1 + \frac{\lambda_1}{\lambda_2 + m}\right)M_t^*.
\end{equation}
And since $f_t(x)$ is $m$-strongly convex, we have
\[ \frac{\lambda_2}{2}\norm{x_t^* - v_t}^2 \leq \frac{\lambda_2}{m}H_t^*.\]
Substituting the above identity into inequality \eqref{T0E2} yields
\begin{equation}\label{T0E3}
    H_t + \lambda_1 M_t + \Delta \phi \leq \frac{m + \lambda_2}{m}H_t^* + \lambda_1\left(1 + \frac{\lambda_1}{m + \lambda_2}\right)M_t^*.
\end{equation}
Using inequality \eqref{T0E3}, we obtain
\begin{equation*}\label{T0E4}
    H_t + M_t + \frac{1}{\lambda_1}\Delta \phi \leq \frac{H_t + \lambda_1 M_t + \Delta \phi}{\lambda_1} \leq \frac{m + \lambda_2}{\lambda_1 m}H_t^* + \left(1 + \frac{\lambda_1}{m + \lambda_2}\right)M_t^*.
\end{equation*}

Theorem \ref{MainT0} follows from summing the above inequality over all timesteps $t$.
\end{proof}

%% file: WG_tradeT1.tex
%\begin{proof}[Proof of Theorem \ref{WG_tradeT1}]
In this proof, we construct counterexamples for two separate cases, based on whether $\lambda_1$ is larger or smaller than $m$. Recall that $\lambda_2=0$ throughout the proof.

\subsubsection*{Case 1: $\lambda_1 > m$}
In this case, we show the competitive ratio can be unbounded by proposing a series of identical hitting cost functions on the real number line. We construct a hitting cost function $f$ with minimizer $v$ so that there exists a fixed point $K \not = v$ (i.e. when $x_{t-1} = K$ and $f_t = f$, the algorithm selects $x_t = x_{t-1}$). Since R-OBD is independent of timestep $t$, we can propose $f_t = f$ for $t = 1, 2, \cdots, T$ and let $x_0 = K$. In this scenario, the total cost of R-OBD grows linearly in $T$. However, by choosing $x_1 = x_2 = \cdots = x_T = v$, the total cost incurred by the offline adversary is a constant. Therefore the competitive ratio of R-OBD will be unbounded.

Specifically, consider the hitting cost function
\[ f(x) = \begin{cases}
\frac{m}{2}\left(1 - (x + 1)^2\right) & -1\leq x\leq 0\\
\frac{m}{2}x^2 & \text{ otherwise}
\end{cases}.\]
Suppose $x_{t-1} = -1$, then R-OBD will choose $x_t$ such that
\[x_t = \argmin_x f(x) + \frac{\lambda_1}{2}(x + 1)^2.\]
Notice that
\[f(x) + \frac{\lambda_1}{2}(x + 1)^2 = \begin{cases}
\frac{m}{2} + \frac{\lambda_1 - m}{2}(x + 1)^2 & -1\leq x\leq 0\\
\frac{m}{2}x^2 + \frac{\lambda_1}{2}(x+1)^2 & \text{ otherwise}
\end{cases}.\]
Since $\lambda_1 > m$, we see that the quantity above is $\ge \frac{m}{2}$ for all real $x$, where equality only holds when $x = -1$.
It follows that $x_t = x_{t-1} = -1 \not = 0 = v$. Thus $K = -1$ is a fixed point satisfying the requirements described as above.

\subsubsection*{Case 2: $\lambda_1 \leq m$}
We consider a situation such that the R-OBD algorithm moves far away from the starting point, incurring significant movement cost, whereas the offline adversary could pay relatively little cost by staying at the starting point. More specifically, suppose the starting point $x_0 = 0$ and the first hitting cost function is $f_1(x) = \frac{m}{2}(1 - x)^2$. 
Consider an adversary which chooses $x_0 = x_1 = \cdots = x_T$. 
The cost incurred by the adversary is
\[ cost(ADV) = \frac{m}{2}. \]
Using the update rule, the R-OBD algorithm chooses
\[ x_1 = \argmin_x \frac{m}{2}(1 - x)^2 + \frac{\lambda_1}{2}x^2 = \frac{m}{m + \lambda_1} \geq \frac{1}{2}. \]
The movement cost incurred by R-OBD is at least
\[ cost(ALG) \ge M_1 = \frac{1}{2}x_1^2 \geq \frac{1}{8}. \]
Thus the competitive ratio is at least
\[ \frac{cost(ALG)}{cost(ADV)} \geq \frac{1}{4m}. \]

Theorem \ref{WG_tradeT1} follows from combining these two cases.
%\end{proof}

%% file: R-OBD-RegretT1.tex
%Now we return to the proof of Theorem \ref{R-OBD-RegretT1}.

%\begin{proof}[Proof of Theorem \ref{R-OBD-RegretT1}]
%We use a similar argument as in the proof of Theorem 10 in \cite{chen2018smoothed}; this approach was also used in \cite{blum1992decomposition} and \cite{blum2000line} \gautam{what are the right citations?}. 
Let $\{x^L_t\}$ be the sequence of points achieving the $L$-constrained offline optimal .
We first prove an upper bound on the difference of hitting costs $f_t(x_t) - f_t(x_t^L)$, and then use this bound to prove a $O\left(G\sqrt{TL}\right)$ upper bound on the regret $\sum_{t=1}^T \left(f_t(x_t) - f_t(x_t^L) + c(x_t, x_{t-1}) \right) - \sum_{t=1}^T c(x_t^L, x_{t-1}^L)$.

Since the function $f_t(x) + \lambda_1 D_h(x || x_{t-1}) + \lambda_2 D_h(x || v_t)$ is strongly convex, it has a unique minimizer, at which point the gradient vanishes. 
This is the point $x_t$ which Algorithm \ref{alg:Breg} picks in round $t$. We can rearrange the vanishing gradient condition to obtain
\[ \nabla f_t(x_t) = \lambda_1 \left( \nabla h(x_{t-1}) - \nabla h(x_t)\right) + \lambda_2 \left( \nabla h(v_t) - \nabla h(x_t)\right). \]

Therefore by Lemma \ref{GWGL1}, we have
\begin{equation}\label{GWGRegret0}
    \begin{aligned}
    \langle \nabla f_t(x_t), x_t - x_t^L \rangle &= \lambda_1 \langle \nabla h(x_{t-1}) - \nabla h(x_t), x_t - x_t^L \rangle + \lambda_2 \langle \nabla h(v_t) - \nabla h(x_t), x_t - x_t^L \rangle\\
    &= \lambda_1 \left(D_h(x_t^L || x_{t-1}) - D_h(x_t^L || x_t) - D_h(x_t || x_{t-1})\right)\\
    &\quad+ \lambda_2 \left(D_h(x_t^L || v_t) - D_h(x_t^L || x_t) - D_h(x_t || v_t)\right).
    \end{aligned}
\end{equation}

Recall that $h$ is $\alpha-$strongly convex and $\beta-$strongly smooth with respect to the norm $\norm{\cdot}$, hence
\begin{equation}\label{GWGRegretE0}
    \forall x, y, \frac{\alpha}{2}\norm{x - y}^2 \leq D_h(x || y) \leq \frac{\beta}{2}\norm{x - y}^2.
\end{equation}
Therefore
\[ D_h(x_t^L || v_t) - D_h(x_t^L || x_t) - D_h(x_t || v_t) \leq D_h(x_t^L || v_t) \leq \frac{\beta}{2}\norm{x_t^L - v_t}^2 \leq \frac{\beta D^2}{2}. \]
In light of equation \eqref{GWGRegret0}, we obtain
\begin{equation} \label{Regret-inequality0}
\langle \nabla f_t(x_t), x_t - x_t^L \rangle \leq \lambda_1 \left(D_h(x_t^L || x_{t-1}) - D_h(x_t^L || x_t) - D_h(x_t || x_{t-1})\right) + \frac{\beta D^2}{2}\cdot \lambda_2.
\end{equation}

Let $q > 0$ be a parameter which we will pick later. For all $q > 0$, it holds that
\begin{subequations}\label{GWGRegret1}
    \begin{align}
    & f_t(x_t) - f_t(x_t^L) \nonumber\\
    \leq{}& \langle \nabla f_t(x_t), x_t - x_t^L \rangle - \frac{m}{2}\norm{x_t - x_t^L}^2\label{GWGRegret1:S1}\\
    \leq{}& \lambda_1 \left(D_h(x_t^L || x_{t-1}) - D_h(x_t^L || x_t) -  D_h(x_t || x_{t-1})\right) - \frac{m}{2}\norm{x_t - x_t^L}^2 + \frac{\beta D^2}{2}\cdot \lambda_2\label{GWGRegret1:S2}\\
    ={}& (\lambda_1 + q) \left(D_h(x_t^L || x_{t-1}) - D_h(x_t^L || x_t)\right) - \lambda_1 D_h(x_t || x_{t-1})\nonumber\\ 
    &- \left(q D_h(x_t^L || x_{t-1}) - q D_h(x_t^L || x_t) + \frac{m}{2}\norm{x_t - x_t^L}^2\right)\nonumber\\
    &+ \frac{\beta D^2}{2}\cdot \lambda_2.\nonumber
    \end{align}
\end{subequations}
where we apply strong convexity in line \eqref{GWGRegret1:S1}, and equation \eqref{Regret-inequality0} in line \eqref{GWGRegret1:S2}.
Using Lemma \ref{GWGL1}, we obtain
\begin{subequations}\label{GWGRegretE2}
    \begin{align}
    &q D_h(x_t^L || x_{t-1}) - q D_h(x_t^L || x_t) + \frac{m}{2}\norm{x_t - x_t^L}^2\nonumber\\
    ={}& q D_h(x_t || x_{t-1}) + q\langle \nabla h(x_{t-1}) - \nabla h(x_t), x_t - x_t^L\rangle + \frac{m}{2}\norm{x_t - x_t^L}^2\nonumber\\
    \geq{}& q D_h(x_t || x_{t-1}) - q\norm{\nabla h(x_{t-1}) - \nabla h(x_t)}_* \norm{x_t - x_t^L} + \frac{m}{2}\norm{x_t - x_t^L}^2\label{GWGRegretE2:S2}\\
    \geq{}& q D_h(x_t || x_{t-1}) - \left(\frac{q^2}{2m}\norm{\nabla h(x_{t-1}) - \nabla h(x_t)}_*^2 + \frac{m}{2}\norm{x_t - x_t^L}^2\right) + \frac{m}{2}\norm{x_t - x_t^L}^2\label{GWGRegretE2:S3}\\
    ={}& q D_h(x_t || x_{t-1}) - \frac{q^2}{2m}\norm{\nabla h(x_{t-1}) - \nabla h(x_t)}_*^2\nonumber\\
    \geq{}& q D_h(x_t || x_{t-1}) - \frac{\beta q^2}{m}D_h(x_t || x_{t-1})\label{GWGRegretE2:S5}\\
    ={}& \left(q - \frac{\beta q^2}{m}\right)D_h(x_t || x_{t-1}),\nonumber
    \end{align}
\end{subequations}
where we apply the Cauchy-Schwartz inequality in line \eqref{GWGRegretE2:S2}, the AM-GM inequality in line \eqref{GWGRegretE2:S3}, and Lemma \ref{DualNormT1} in line \eqref{GWGRegretE2:S5}. 

In order to maximize the coefficient $\left(q - \frac{\beta q^2}{m}\right)$, we set $q = \frac{m}{2\beta}$. By substituting inequality \eqref{GWGRegretE2} into inequality \eqref{GWGRegret1}, we obtain
\begin{equation}\label{GWGRegretE3}
\begin{aligned}
    &f_t(x_t) - f_t(x_t^L) \\
    \leq{}& \left(\lambda_1 + \frac{m}{2\beta}\right) \Big(D_h(x_t^L || x_{t-1}) - D_h(x_t^L || x_t)\Big) - \left(\lambda_1 + \frac{m}{4\beta}\right) D_h(x_t || x_{t-1}) + \frac{\beta D^2}{2}\cdot \lambda_2.
\end{aligned}
\end{equation}
Using the condition $\lambda_1 + \frac{m}{4\beta} \geq 1$, we observe that
\begin{equation}\label{GWGRegretE4}
    f_t(x_t) - f_t(x_t^L) + D_h(x_t || x_{t-1}) \left(\lambda_1 + \frac{m}{2\beta}\right) \Big(D_h(x_t^L || x_{t-1}) - D_h(x_t^L || x_t)\Big) + \frac{\beta D^2}{2}\cdot \lambda_2.
\end{equation}
Notice that
\begin{equation}\label{GWGRegretE5}
    \sum_{t=1}^T \norm{x_t^L - x_{t+1}^L} \leq \sqrt{T\left( \sum_{t=1}^T \norm{x_t^L - x_{t+1}^L}^2 \right)} \leq \sqrt{T\left( \sum_{t=1}^T \frac{2D_h(x_{t+1}^L || x_t^L)}{\alpha} \right)} \leq \sqrt{\frac{2TL}{\alpha}}.
\end{equation}
where we use the generalized mean inequality in the first step and $\alpha$-strong convexity of $h$ in the second step (cf. equation \eqref{GWGRegretE0}). By Lemma \ref{GWGL2}, we can give the following upper bound:
\begin{subequations}\label{GWGRegretE6}
    \begin{align}
    &\sum_{t=1}^T D_h(x_t^L || x_{t-1}) - D_h(x_t^L || x_t)\nonumber\\
    ={}& \sum_{t=1}^T \left( D_h(0 || x_{t-1}) - D_h(0 || x_t) + \langle \nabla h(x_t) - \nabla h(x_{t-1}), x_t^L\rangle \right)\nonumber\\
    ={}& D_h(0 || x_0) - D_h(0 || x_T) + \sum_{t=1}^{T-1}\langle \nabla h(x_t), x_t^L - x_{t+1}^L\rangle - \langle \nabla h(x_0), x_1^L\rangle + \langle \nabla h(x_T), x_T^L\rangle\nonumber\\
    \leq{}& \sum_{t=1}^T \langle \nabla h(x_t), x_t^L - x_{t+1}^L \rangle \label{GWGRegretE6:S3}\\
    \leq{}& \sum_{t=1}^T \norm{\nabla h(x_t)}_* \norm{x_t^L - x_{t+1}^L} \label{GWGRegretE6:S4}\\
    \leq{}& G\sum_{t=1}^T \norm{x_t^L - x_{t+1}^L}\nonumber\\
    \leq{}& G\sqrt{\frac{2TL}{\alpha}},\label{GWGRegretE6:S6}
    \end{align}
\end{subequations}
where we use the facts $x_0 = x_0^L = x_{T+1}^L = 0, \nabla h(0) = 0$ in line \eqref{GWGRegretE6:S3}, the Cauchy-Schwartz inequality in line \eqref{GWGRegretE6:S4}, and inequality \eqref{GWGRegretE5} in line \eqref{GWGRegretE6:S6}.

Therefore we obtain
\begin{subequations}\label{GWGRegretE7}
    \begin{align}
    &cost(OBD) - cost(OPT(L))\nonumber\\
    ={}& \sum_{t=1}^T\left(f_t(x_t) + D_h(x_t || x_{t-1})\right) - \left(f_t(x_t^L) + D_h(x_t^L || x_{t-1}^L)\right)\label{GWGRegretE7:S1}\\
    \leq{}& \left(\sum_{t=1}^Tf_t(x_t) - f_t(x_t^L) + D_h(x_t || x_{t-1})\right) - L\nonumber\\
    \leq{}& \left(\lambda_1 + \frac{m}{2\beta}\right)G\sqrt{\frac{2TL}{\alpha}} + T\cdot \frac{\beta D^2}{2}\cdot \lambda_2 - L,\label{GWGRegretE7:S3}
    \end{align}
\end{subequations}
where we use the definition of $OPT(L)$ in line \eqref{GWGRegretE7:S1}; inequalities \eqref{GWGRegretE4} and \eqref{GWGRegretE6} in line \eqref{GWGRegretE7:S3}.

Since by assumption we have $G < \infty$, $\lambda_2 = \eta(T, L, D, G) \leq \frac{K G}{D^2}\cdot \sqrt{\frac{L}{T}}$ for some constant $K$, by inequality \eqref{GWGRegretE7}, we obtain
\[ cost(OBD) - cost(OPT(L)) = O(G\sqrt{TL}), \]
which completes the proof.

%\end{proof}